\documentclass[twoside]{article}
\usepackage[accepted]{aistats2022}

\usepackage{amsmath,amsthm,amssymb,dsfont}
\usepackage{graphicx}
\usepackage{xfrac}
\usepackage[inline]{enumitem}
\usepackage{subfigure}
\usepackage[ruled,linesnumbered]{algorithm2e}
\usepackage{accents}
\usepackage{caption}
\usepackage{float}
\usepackage{pgf,tikz}
\usetikzlibrary{calc,shapes,positioning,arrows,arrows.meta}
\usepackage{natbib}
\usepackage{comment}
\usepackage{mdframed}

\newtheorem{thmdef}{Definition}

\newtheorem{thmprop}{Proposition}

\newtheorem{thmappdef}{Definition}

\newtheorem{thmappasmp}{Assumption}

\newtheorem{thmapplem}{Lemma}

\newtheorem{thmappprop}{Proposition}

\def\bu{\mathbf{u}}

\def\bw{\mathbf{w}}

\def\E{\mathbb{E}}
\def \R{\mathbb{R}}

\def\cP{\mathcal{P}}

\def\cN{\mathcal{N}}

\def\cD{\mathcal D}

\def\cH{\mathcal H}

\def\cF{\mathcal F}
\def\cZ{\mathcal Z}

\def\cZ{\mathcal Z}

\def\bx{\mathbf{x}}
\def\bX{\mathbf{X}}
\def\bmu{\boldsymbol{\mu}}

\def\mmd{\mathrm{MMD}}
\def\mmdphi{\mmd(P^\circ_{\boldsymbol{\phi}_0}, P^\circ_{\boldsymbol{\phi}_1})}

\newcommand\indep{\protect\mathpalette{\protect\independenT}{\perp}}
\def\independenT#1#2{\mathrel{\rlap{$#1#2$}\mkern2mu{#1#2}}}

\newcommand\damour[1]{\textcolor{blue}{[AD: #1]}}

\newcommand\rinv{\text{rinv}}

\usepackage{booktabs}
\usepackage[breaklinks=true]{hyperref}
\begin{document}
\runningauthor{Maggie Makar, Ben Packer, Dan Moldovan, Davis Blalock, Yoni Halpern, Alexander D'Amour}


\twocolumn[

\aistatstitle{Causally Motivated Shortcut Removal Using Auxiliary Labels}
% \footnote{Part of this work was completed while MM was an intern at Google Research}
\aistatsauthor{Maggie Makar
\And Ben Packer 
\And  Dan Moldovan}
\aistatsaddress{ University of Michigan, Ann Arbor
\And  Google Research
\And Google Research}
\aistatsauthor{
Davis Blalock 
\And Yoni Halpern
\And Alexander D'Amour
}
\aistatsaddress{
MosaicML
\And Google Research
\And Google Research} ]

\begin{abstract}
    Shortcut learning, in which models make use of easy-to-represent but unstable associations, is a major failure mode for robust machine learning.
    We study a flexible, causally-motivated approach to training robust predictors by discouraging the use of specific shortcuts, focusing on a common setting where a robust predictor could achieve optimal \emph{iid} generalization in principle, but is overshadowed by a shortcut predictor in practice.
    Our approach uses auxiliary labels, typically available at training time, to enforce conditional independences implied by the causal graph.
    We show both theoretically and empirically that causally-motivated regularization schemes (a) lead to more robust estimators that generalize well under distribution shift, and (b) have better finite sample efficiency compared to usual regularization schemes, even when no shortcut is present.
    Our analysis highlights important theoretical properties of training techniques commonly used in the causal inference, fairness, and disentanglement literatures. 
    % \url{github.com/mymakar/causally_motivated_shortcut_removal}.
\end{abstract}

\section{INTRODUCTION}

Despite their immense success, predictors constructed from deep neural networks (DNNs) have been shown to lack robustness under distribution shift \citep{beery2018recognition,ilyas2019adversarial,azulay2018deep,geirhos2018generalisation}, especially naturally occurring distribution shifts \citep{taori2020measuring}. 
One mechanism for this brittleness is \emph{shortcut learning} \citep{geirhos2020shortcut}.
Shortcut learning occurs when a predictor relies on input features that are easy to represent (i.e., shortcuts) and are predictive of the label in the training data, but whose association with the label changes under relevant distribution shifts.
For example, a DNN trained for image classification could exploit correlations between the foreground object and background of images in the training distribution, and use a representation of the background as a shortcut to predict the foreground object \citep{beery2018recognition,sagawa2019distributionally}. This holds even if the foreground object alone is sufficient to achieve optimal predictive performance on the training distribution \citep{nagarajan2020understanding,pmlr-v119-sagawa20a}. 
\begin{figure}
\centering
\resizebox{0.35\columnwidth}{!}{%
\begin{tikzpicture}[var/.style={draw,circle,inner sep=0pt,minimum size=0.8cm}]
    % Nodes
    \node (Xstar) [var] {$\bX^*$};
    \node (input) [var, below=0.75cm of Xstar, fill=black!10] {$\bX$};
    \node (label) [var, fill=black!10, right=0.5cm of Xstar] {$Y$};
    \node (metadata) [var, fill=black!10, right=0.5cm of input] {$V$};
    
    % Arrows
    \path[->]
        (label) edge (Xstar)
        (Xstar) edge (input)
        (metadata) edge (input);
    \path[<->,dashed]
        (label) edge (metadata);
\end{tikzpicture}
}
\caption{Causal DAG of the setting in this paper. The main label $Y$ and auxiliary label $V$ generate observed input $\bX$, but $Y$ only affects $\bX$ through $\bX^*$.\label{fig:dags}}
\vspace{-1.5em}
\end{figure}
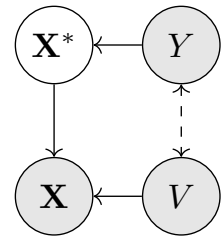

%Throughout the paper, we use the example of classifying the foreground object as land or water bird, adapted from \citep{sagawa2019distributionally}. The two classes are visually distinct but the majority of the former often appear on land backgrounds, and the latter on water backgrounds. DNNs that exploit shortcuts could achieve strong performance on unseen instances from the training distribution, but would fail if the foreground object and background were correlated differently in the test distribution (e.g., if water birds appeared on land backgrounds).
%More generally, models that rely on shortcuts are vulnerable to shifts in distribution induced by intervening on factors of variation that are correlated---but not causally related---with the main label in the training distribution.

Here, we consider the problem of learning a performant predictor whose risk is invariant to 
interventions that change
%changes to
the association between irrelevant factors and the main label. 
%\damour{In previous sentence, we probably don't want to refer to latent factors anymore.}
Ideally, such a predictor would rely exclusively on input features that are invariant to irrelevant factors. 
However, identifying such invariant input features in the standard supervised learning setup is difficult, for the same reason that shortcut learning is successful: in learning setups where there are many distinct ways to construct predictors that perform well on held-out data (i.e., when the learning problem is \emph{underspecified} \citep{damour2020underspecification}), the influence of correlated factors is difficult to disentangle without additional supervision \citep{locatello2019challenging}.

For this reason, we focus on a setting where we are also given an auxiliary label that gives information about the irrelevant factor at training time.
Such labels often appear in the form of metadata for the training data---for example, labels of the background of an image---but are often not available at test time.
We propose an approach that exploits this auxiliary label to construct a predictor whose risk is approximately invariant across a well-defined family of test-distributions.
Our method makes use of two tools from causal inference in combination: (1) weighting the training data to mimic an idealized population, and (2) enforcing an independence implied by the causal Directed Acyclic Graph (DAG) in that idealized population.
While each of these approaches has been applied separately, we show here through both theoretical arguments and empirical analysis that these methods are particularly effective when applied together. 

Our methodological contributions can be summarized as follows: 
\begin{enumerate*}[label={(\arabic*)}]
    \item We suggest an approach to discourage shortcut learning  using auxiliary labels, and specify a set of distribution shifts across which a robust model is risk-invariant.
    \item We give a theoretical justification to our approach, highlighting that in some scenarios it yields models that have a lower generalization error than typical regularization schemes. 
    %\damour{can we use more precise language here?}
    We also show that our approach is robust to a set of distribution shifts. 
    \item We empirically validate our theoretical findings using a semi-simulated benchmark and a medical task, showing our approach has favorable in- and out-of-distribution generalization properties.  
    % \item We compare against baselines that ablate each part of our approach to show that their combination yields more performant, stable training. 
    %\item We highlight important connections to existing work in fairness, and causality. 
    %\item We demonstrate the wide applicability of our approach to many existing datasets.
\end{enumerate*}

%The remainder of the paper is organized as follows. In section~\ref{sec:prelims}, we formally introduce our objective. We also discuss important properties of the unconfounded distribution, where the main label and the auxiliary label are independent. In section~\ref{sec:summary_approach}, we present our main approach, and briefly state the main claims that guide the design of our approach. We revisit these claims in section~\ref{sec:theory} with a greater detail, giving theoretical justification for each. We present our empirical analysis in section~\ref{sec:experiments}, and conclude in section~\ref{sec:conclusion}.

\section{PRELIMINARIES}\label{sec:prelims}

\textbf{Setup.} 
We consider a supervised learning setup where the task is to construct a predictor $f(\bX)$ parameterized by weights $\bw$ that predicts a label $Y$ (e.g., foreground object) from an input $\bX$ (e.g., image).
In addition, we have an auxiliary label $V$ (e.g., background label) available only at training time that labels a factor of variation along which we hope the model will exhibit some invariance (e.g., background type).
Throughout, we will use capital letters to denote variables, and small letters to denote their value. 
Our training data consist of tuples $\mathcal D = \{(\bx_i, y_i, v_i)\}_{i=1}^{n}$ drawn from a source training distribution $P_s$.
We restrict our focus to the case where $Y$ and $V$ are binary and $f$ is a classifier.
Specifically, we will consider functions $f$ of the form $f = h(\phi(\bx))$, where $\phi$ is a representation mapping and $h$ is the final classifier. 

\textbf{Assumptions.} We assume that $P_s$ has a generative structure shown in Figure~\ref{fig:dags}, in which the inputs $\bX$ are generated by the labels $(Y, V)$.
We assume that the labels $Y$ and $V$ are are correlated, but not causally related; that is, an intervention on $V$ does not imply a change in the distribution of $Y$, and vice versa.
Such correlation often arises through the influence of an unobserved third variable such as the environment from which the data is collected. 
We represent this in Figure~\ref{fig:dags} with the dashed bidirectional arrow.

We assume that there is a sufficient statistic $\bX^*$ such that $Y$ only affects $\bX$ through $\bX^*$, and $\bX^*$ can be fully recovered from $\bX$ via the function $\bX^* := e(\bX)$.
However, we assume that the sufficient reduction $e(\bX)$ is unknown, so we denote $\bX^*$ as unobserved in Figure~\ref{fig:dags}.

We also make an overlap assumption on the source distribution, $P_s$. Specifically we assume that $P_s(V)P_s(Y) \ll P_s(V, Y)$.
%\vspace{-1em}
\subsection{Risk Invariance and Shortcut Learning}
% \vspace{-1em}
We define the generalization risk of a function $f$ on a distribution $P$ as $R_P = \E_{X, Y \sim P}[\ell(f(X), Y)]$, where $\ell$ is the logistic loss. 

We focus on obtaining an optimal \emph{risk invariant} predictor, whose risk is invariant across a family of target distributions $\cP$ that can be obtained from $P_s$ by interventions on the causal model in Figure~\ref{fig:dags}.
Specifically, we consider interventions on the confounding relationship between $Y$ and $V$ that keep the marginal distribution of $Y$ constant.
Each distribution in this family can be obtained by replacing the source conditional distribution $P_s(V \mid Y)$ with a target conditional distribution $P_t(V \mid Y)$:
\begin{align}
\cP = \{P_s(\bX \mid \bX^*, V) P_s(\bX^* \mid Y) P_s(Y) P_t(V \mid Y)\},\label{eq:intervention family}
\end{align}
%where $\ll$ denotes ``absolutely continuous with respect to,'' and indicates that the support of $P_t$ is contained in the support of $P_s$.
This family allows the marginal dependence between $Y$ and $V$ to change arbitrarily.
%Within the family of distributions $\mathcal P$, we pay special attention to the \emph{unconfounded distribution} $P^\circ \in \cP$ where $P_t(V \mid Y) := P_s(V)$.
%Under $P^\circ$, $Y \indep V$ and the dashed bidirectional arrow in Figure~\ref{fig:dags} can be dropped.

Given the family $\cP$, we define the set of risk invariant predictors to be all predictors that have the same risk for all $P_t \in \cP$,
$
\cF_{\rinv} = \{f : R_{P_t}(f) = R_{P_t'}(f) \quad\forall P_t, P_t' \in \cP\}
$
and an optimal risk-invariant predictor $f_{\rinv}$ to have the property
$
f_{\rinv} \in \arg\min_{f \in \cF_{\rinv}} R_{P_t}(f) \quad\forall P_t \in \cP.
$

Any predictor that does not rely on the ``shortcut'' factor labeled by $V$  will necessarily be risk invariant; thus, it is a natural property to test for and enforce for shortcut removal.
Risk invariance is also independently appealing because guarantees about the performance of the predictor $f_\rinv$ derived under one distribution can be adapted to other distributions in $\mathcal P$.
% We discuss relationships to other notions of robustness and invariance in section~\ref{sec:literature review}.
% \makar{Do we?}
%\vspace{-1em}
\subsection{The Unconfounded Distribution $P^\circ$}
%\vspace{-1em}
Within the family of distributions $\mathcal P$, we pay special attention to the \emph{unconfounded distribution} $P^\circ \in \cP$ where $P^\circ(V \mid Y) := P_s(V)$.
Under $P^\circ$, $Y \indep V$ and the dashed bidirectional arrow in Figure~\ref{fig:dags} can be dropped.
Both our methodological approach and theoretical analysis revolve around mapping the problem of learning a risk invariant predictor under $P_s$ to the problem of learning an optimal predictor under $P^\circ$.

$P^\circ$ has two useful properties that are revealed by the DAG in Figure~\ref{fig:dags}: (1) under the unconfounded distribution $P^\circ$, the optimal predictor (with some abuse of notation) would take the form $f(\bX^*)$, and (2) for any predictor of the form $f(\bX^*)$, the joint distribution $P(f(\bX^*), Y)$ (and thus the risk) is invariant across the family $\cP$.
Together, these imply that the optimal risk-invariant predictor $f_\rinv(\bX^*)$ is the optimal predictor under $P^\circ$. We state this formally in Proposition~\ref{prop:frob_ferm_equal}.
\begin{thmprop}\label{prop:frob_ferm_equal}
Under $P^\circ$, the Bayes optimal predictor is (i) only a function of $\bX^*$, and (ii) an optimal risk-invariant predictor $f_{\rinv}$ with respect to $\cP$.
\end{thmprop}
All proofs are shown in the appendix. 

Proposition~\ref{prop:frob_ferm_equal} motivates our approach that aims to efficiently estimate the optimal predictor under $P^\circ$, even when the training data $\mathcal D$ are drawn from $P_s \not = P^\circ$.

\section{APPROACH}\label{sec:summary_approach}

Here, we describe our approach to learning an optimal risk-invariant predictor $f_\rinv(\bX)$ from training data $\cD \sim P_s$.
Our strategy follows the schematic $P_s \rightarrow P^\circ \rightarrow P_t$: to learn a risk invariant predictor under $P_s$, we map the learning problem to $P^\circ$, then learn a predictor under $P^\circ$ that will generalize well in finite samples to any $P_t$ in $\cP$.
The $P_s \rightarrow P^\circ$ step is achieved by importance weighting.
%, which is relatively familiar to a wider audience.
The $P^\circ \rightarrow P_t$ step is achieved by a tailored regularization scheme, which we present first.

\textbf{Regularization under $P^\circ$.}
Our main innovation is a regularizer that leverages knowledge of the causal DAG to efficiently learn a classifier under $P^\circ$.
Specifically, we use two facts that hold under $P^\circ$: (1) $V \indep \bX^*$, and (2) the optimal predictor is only a function of $\bX^*$ (Proposition~\ref{prop:frob_ferm_equal}). 
As we show in Section~\ref{sec:theory}, this regularizer can help to reduce the sample complexity of learning without inducing bias, and directly penalizes the gap between the risk under $P^\circ$ and $P_t$.

Based on these facts, we specify a regularizer for $f = h(\phi(\bX))$ that encourages $\phi(\bX) \indep V$.
We do this by penalizing a distributional discrepancy between conditional distributions of the representation $P^\circ(\phi(\bX) \mid V = 0)$ and $P^\circ(\phi(\bX) \mid V = 1)$ that would be identical under independence.
Although any number of estimable distributional discrepancy metrics could be used, here we choose to use the Maximum Mean Discrepancy ($\mmd$), defined as follows:
\begin{thmdef}\label{def:mmd}
Let $Z \sim P_{Z}$, and $Z' \sim P_{Z'}$, be two arbitrary variables. And let $\Omega$ be a class of functions $\omega: \cZ \rightarrow \R$,  
\begin{align*}
    \mmd(\Omega, P_{Z}, P_{Z'}) = \underset{\omega \in \Omega}{\sup} \big(\E_{P_{Z}} \omega(Z) - \E_{P_{Z'}} \omega(Z')\big).
\end{align*}
\end{thmdef}
\vspace{-1em}
When $\Omega$ is set to be a general reproducing kernel Hilbert space (RKHS), the $\mmd$ defines a metric on probability distributions, and is equal to zero if and only if $P_{Z} =P_{Z'}$. 
Throughout, we will assume that our predictor $f$ and our loss $\ell$ are contained in $\Omega$, and in practice choose $\Omega$ to be the RKHS induced by the radial basis function (RBF) kernel. We will use the shorthand $\mmd(P_{Z}, P_{Z'})$ to denote $\mmd(\Omega, P_{Z}, P_{Z'})$. 
See \citet{gretton2012kernel} for a review of $\mmd$ and its empirical estimators.

\textbf{Weighting to Recover $P^\circ$.}
When the training data is drawn from some $P_s \not= P^\circ$, we weight the data to obtain empirical risk and $\mmd$ expressions that are unbiased estimates of the expressions we would obtain if $\cD \sim P^\circ$, and proceed as before.
We define weights

\begin{align}
u(y,v) =\frac{P_s(Y = y)P_s(V = v)}{P_s(Y = y, V = v)}, \label{eq:weights}
\end{align}
such that for each example, $u_i := u(y_i, v_i)$.
We use $\tilde{u}_i$ to denote $u_i$ after normalization such that $\sum_i \tilde{u}_i = 1$. 
For any distribution $P_s$, these are importance weights that map expectations under $P_s$ to expectations under $P^\circ$. In the appendix, we show that the reweighted risk is an unbiased estimator of the risk under $P^\circ$, i.e., that 
	$
	    \E_{Ps}\left[\hat{R}^\bu_{Ps}(f)\right] = R_{P^\circ}(f), 
	$
	where $\hat{R}^\bu_{Ps}(f) = \sum_i \tilde{u}_i \ell(f(\bx_i), y_i)$, and $R_{P^\circ}(f) = \E_{\bX, Y \sim P^\circ}\left[ \ell(f(\bX), Y) \right]$. 
	
\textbf{Method.}
Our final objective combines these weighting and regularization components.
Both components of our approach are necessary together to achieve invariance and efficiency: $\mmd$ regularization without reweighting would lead to biased estimation/inaccurate models (proposition \ref{prop:looser_mmd} in the appendix). Reweighting alone would remove the dependence on the shortcut in the asymptotic regime (proposition \ref{prop:frob_ferm_equal}), but as is typical for reweighted estimators, it leads to high variance, which is improved by using the $\mmd$ regularizer.

Let $\boldsymbol{\phi}_v$ denote $\{ \phi(\bx_i) \}_{i: v_i = v}$,
%and $\boldsymbol{\phi}^\bu_v$ denote its re-weighted analogue,
and let $u_i$ be as in equation~\ref{eq:weights},
and $P_{\boldsymbol \phi_v}^\bu$ be the weighted distribution of $\phi_v$.
For $\cD \sim P_s$, and some $\alpha > 0$, the main objective to minimize is:
\vspace{-0.5em}
\begin{align}\label{eqn:empirical_mmd_loss}
    h^*, \phi^*  & = \underset{h, \phi}{\text{argmin}} \sum_i u_i \ell(h(\phi(\bx_i)), y_i) \\
    & + \nonumber \alpha \cdot \widehat\mmd^2 (P_{\boldsymbol{\phi}_0}^{\bu}, P_{\boldsymbol{\phi}_1}^\bu). 
\end{align}
Here, $\widehat \mmd^2$, is a weighted version of the V-statistic estimator presented in~\cite{gretton2012kernel}. Specifically, we compute: 
\vspace{-0.5em}
\begin{align*}
    & \widehat \mmd^2  = 
    \sum_{i, j:v_i, v_j = 0} \overline{u}_i \overline{u}_j k_\gamma(\phi_i, \phi_j) 
    \\
    & +\sum_{i, j:v_i, v_j = 1} \overline{u}_i \overline{u}_j k_\gamma(\phi_i, \phi_j) 
    - 2 \sum_{\substack{i, j:v_i=0,\\v_j = 1}} \overline{u}_i \overline{u}_j k_\gamma(\phi_i, \phi_j), 
\end{align*}
where $k_\gamma(x, x')$ is the radial basis function, with bandwidth $\gamma$, and the weights $\overline{u}_i$ are a normalized version of $u_i$ such that $\sum_{i, j:v_i, v_j = 0} \overline{u}_i \overline{u}_j = \sum_{i, j:v_i, v_j = 1} \overline{u}_i \overline{u}_j = \sum_{i, j:v_i=0, v_j = 1} \overline{u}_i \overline{u}_j = 1$.

\paragraph{Cross-validation.}
The objective function in~\eqref{eqn:empirical_mmd_loss} depends on two hyperparameters: the cost of the $\mmd$ penalty $\alpha$, and the penalty's kernel bandwidth $\gamma$.
Unlike many regularizers, the $\mmd$ penalty depends on the distribution of the data, and is vulnerable to overfitting, such that the estimated $\mmd$ on the training data underestimates the population $\mmd$.
For this reason, we follow a two-step cross-validation procedure.
In the first step, we calculate the weighted MMD on each of the $K$ validation folds.
We test if the achieved $\mmd$s are statistically significantly different from zero using a T-test and exclude all models with a p-value $< 0.05$.
%This gives us a subset of the function candidates that encode the desired invariances.
In the second step, we pick the best performing model out of this subset of candidate functions. 

\section{THEORY}\label{sec:theory}

We analyze our approach in some important special cases to show how it encourages efficient learning of a risk-invariant predictor on $\cP$.
We focus on the generalization gap of a predictor learned under a source distribution $P_s$ when evaluated on any target distribution $P_t \in \cP$,
We define this gap as the difference between the target risk $R_{P_t}$ and the weighted empirical risk $\hat R^{\bu}_{P_s}$ for a learned function $f$.
In keeping with our schematic $P_s \rightarrow P^\circ \rightarrow P_t$ (described in Section~\ref{sec:summary_approach}), we decompose the gap into a term about generalizing from $P_s$ to $P^\circ$, and a term about generalizing from $P^\circ$ to $P_t$:
\begin{align}\label{eq:overall_goal}
    R_{Pt}(f) - \hat{R}^\bu_{Ps}(f) & = \underbrace{R_{Pt}(f) - R_{P^\circ}(f)}_{\text{Structural risk gap }(P^\circ \rightarrow P_t)} 
    \\ \nonumber & + \underbrace{R_{P^\circ(f)} - \hat{R}^\bu_{Ps}(f)}_{\text{Learning gap }(P_s \rightarrow P^\circ)} 
\end{align}
We show that constraints on the $\mmd$ between the learned representation distributions $P^\circ(\phi(\bX) \mid V=v)$ for $v \in \{0,1\}$, which we denote $\mmdphi$, can translate to constrains on both of these gaps.
We treat the structural risk gap first, but spend the bulk of our effort on the finite-sample gap in the linear case.

\subsection{Bounding the structural risk gap}\label{sec:struc_gap}
% \vspace{-1em}
We begin by bounding the structural risk gap, $R_{Pt}(f) - R_{P^\circ}(f)$,  which characterizes how the risk for a given function $f$ learned on data from $P^\circ$ is related to the risk on any target distribution $P_t \in \cP$.
Here, we show that a bound on $\mmdphi$ translates directly to a bound on the structural risk gap in the representable case, when $Y$ can be written as a function of the representation $\phi(\bX)$.

\begin{thmprop}\label{prop:risk_gap}

Suppose that $f = h(\phi(\bX))$ is a predictor that satisfies $\mmdphi \leq \tau$.
Suppose that $y$ is $\phi-$representable, i.e., that there exists $g(\phi(\bx)) = y$, and that $g(\phi)\ell(\phi) \in \Omega$. 
Then
$
    R_{Pt}(f) < R_{P^\circ}(f) + 2 \tau.
$
\end{thmprop}

Proposition~\ref{prop:risk_gap} states that by encouraging small values of $\tau$, the $\mmd$ penalty regularizes the solution toward a predictor that has similar risks on $P_t$ and $P^\circ$, reducing the gap in equation \eqref{eq:overall_goal}.
%Proposition~\ref{prop:risk_gap} motivates the $\mmd$ penalty. The $\mmd$ penalty directly encourages small values of $\tau$; this regularizes the solution toward a predictor that has similar risks on $P_t$, and $P^\circ$. This in turn means that the first term in equation~\ref{eq:overall_goal} is small, leading to low generalization error of our proposed weighted estimator.  
%\vspace{-8mm}
\subsection{Bounding the learning gap~\label{sec:finite_gap}}
% \vspace{-3mm}
We now analyze how the $\mmd$ regularizer constrains the learning gap, $R_{P^\circ}(f) - \hat{R}^\bu_{P_s}(f)$.
This gap characterizes how the weighted empirical risk, which is minimized during training on data drawn from $P_s$, translates to risk on the ideal distribution $P^\circ$.
Here, we consider the special case of linear models, where $\phi$ is a linear mapping, i.e., $\phi(\bx) = \bw^\top \bx$, and $h$ is the sigmoid, i.e., $h(x)= \sigma(x) = \sfrac{1}{\left(1 + \exp(-x)\right)}$. 
Our analysis establishes some key insights about how the $\mmd$ regularizer constrains the learning problem, and when we expect it to provide significant efficiency gains.
Extensions of our theoretical analysis to more complex neural networks are possible (e.g., through approaches studied in \cite{golowich2018size}). 

\textbf{Controlling complexity.}
There are two issues to consider in bounding this gap.
The first is the fundamental generalization gap that we would face if we were learning $f$ directly on ``ideal'' samples from $P^\circ$, and the second is the additional variability incurred by importance weighting to translate the learning problem to $P^\circ$ to $P_s$, from which we actually observe samples.
Regularization affects both the ``ideal'' and weighted learning problems in the same way, by restricting the complexity of the function class $\cF$.
Here, we focus on the Rademacher complexity of the function class obtained by constraining $\mmdphi$. 

\begin{thmdef}
    Let $\boldsymbol{\epsilon} = \{ \epsilon_i\}_{i=1}^n$ denote a vector of independent random variables drawn from the Rademacher distribution, i.e., uniform on $\{-1,1\}$. For a function family $\cF$, and $\cD \sim P$, the Rademacher complexity for a sample of size $n$ is defined as: 
    $
        \mathfrak{R}(\cF) = \E_\cD\E_{\boldsymbol{\epsilon}} \bigg[ \underset{f \in \cF}{\sup} \frac{1}{n}\sum_{i=1}^n\epsilon_i f(\bx_i) \bigg].
   $
\end{thmdef}

In the ideal case, where $P_s = P^\circ$, the Rademacher complexity translates directly to a bound on the learning gap. Specifically, the learning gap increases as  $\mathfrak{R}(\cF)$ increases (see \citep{mohri2018foundations}). 
\begin{comment}
For a fucntion class $\cF$ of bounded functions, a loss function that is L-Lipschitz, and a training data of size $n$, with probability $1-\delta$, the following holds \citep{mohri2018foundations}:
\begin{align}\label{eq:GE_standard_formula}
    R_{P^\circ}(f) \leq \hat{R}_{P^\circ}(f) + L \cdot \mathfrak{R}(\cF) + \sqrt{\frac{\log \frac{1}{\delta}}{2n}}.  
\end{align}
\end{comment}
For the case where $P_s \not= P^\circ$, and weighting is necessary, we can obtain a bound that is similarly monotonic in $\mathfrak{R}(\cF)$ using the technique of \citet{cortes2010learning}. Due to space limitations, we focus on analyzing the Rademacher complexity of the $\mmd$ regularized space in the main text and present the full generalization error bounds for both cases (when $P_s = P^\circ$ and when $P_s \not= P^\circ$) in the appendix. % section~\ref{sec:full_ge_statements}.

\textbf{Comparison with L2 Regularization}
To analyze how constraints on $\mmdphi$ affect Rademacher complexity, we study how the addition of an $\mmd$ constraint shrinks a standard $L2$-regularized function class. Define the two function classes:
\begin{align}
& \cF_{L_2}  := \{f: \bx \mapsto \sigma(\bw^\top \bx), \|\bw\|_2 \leq A\},\label{eq:L2 class} \\
& \cF_{L_2, \mmd} := \{f: \bx \mapsto \sigma(\bw^\top \bx), \|\bw\|_2 \leq A, \label{eq:mmd_class}
\\ \nonumber & \qquad \qquad  \mmd(P^\circ_{\boldsymbol{\phi}_0}, P^\circ_{\boldsymbol{\phi}_1}) \leq \tau\} .
\end{align}
Before analyzing the precise reduction in complexity in moving from $\cF_{L_2}$ to $\cF_{L_2, \mmd}$, we note that this reduction is a ``free lunch'':
because we know that the $\mmd$ constraint is compatible with the true independence structure of $P^\circ$, this reduction does not introduce bias. 
%Proposition~\ref{prop:subsets} states that even in the absence of distribution shift, when $P_s = P^\circ$, explicitly penalizing $\mmd$ is advantageous because it reduces the hypothesis space without introducing bias.
%Since $\cF_{L_2, \mmd} \subseteq \cF_{L_2}$, we expect the $\mmd$ penalty to reduce the hypothesis space.
%However the key thing to note here is that this reduction does not introduce bias.
%We state this formally in the following proposition. 
We formalize this in Proposition~\ref{prop:subsets}.
\begin{thmprop}\label{prop:subsets}
Under $P^\circ$, and for $\cF_{L_2}$ such that $f_{\text{rinv}} \in \cF_{L_2}$, there exists $\cF_{L_2, \mmd} \subseteq \cF_{L_2}$ such that $f_{\text{rinv}} \in \cF_{L_2, \mmd}$. And the smallest $\cF_{L_2, \mmd}$ such that $f_\rinv \in \cF_{L_2, \mmd}$ has $\tau=0$.
\end{thmprop}

To study the reduction in complexity, we construct comparable Rademacher complexity upper bounds for $\cF_{L_2}$ and $\cF_{L_2, \mmd}$.
%Before delving into the comparison between those two function classes, we highlight a core finding that provides intuition about the advantages of the $\mmd$ penalty.
We focus on a specific implication of the $\mmd$ constraint in the case where $\phi$ is linear.
%In the special case where $\phi$ is a linear mapping, the $\mmd$ constraint has direct implications for the weights $\bw$.
Here, the constraint restricts the projection of the weights $\bw$ onto the vector that distinguishes the conditional means: $\Delta := \bmu_0 - \bmu_1$, where $\bmu_v := \E_{\bx \sim P^\circ}[\bX \mid V=v]$.
$\Delta$ is the average change in $\bX$ caused by intervening to change the $V$ under $P^\circ$.
Define the projection matrix  $\Pi := \Delta (\Delta^\top \Delta)^{-1} \Delta^\top = \|\Delta\|^{-2}_2 \Delta \Delta^\top$, which projects any vector onto $\Delta$, and $\bw_\perp := \Pi \bw$ as the projection of $\bw$ onto the mean distinguishing vector $\Delta$, which can be thought of as the ``irrelevant'' dimension of $\bX$. We can directly relate $\|\bw_\perp\|$ to the $\mmd$ constraint, in the following proposition. 

\begin{thmprop} \label{prop:wdelta_bound}
Let $f(\bx) = \sigma(\phi(\bx)) = \sigma(\bw^\top \bx)$ be a function contained in $\cF_{L_2, \mmd}$. Then,
$
\|\bw_\perp\| \leq \frac{\tau}{\|\Delta\|}.
$
\end{thmprop}
Proposition~\ref{prop:wdelta_bound} says that the $\mmd$ constraint limits the effect of the irrelevant components of $\bw$ proportionally to $\tau$. In the image classification example, this means that the parts of $\bw$ that discriminate between images with different background types is constrained by $\tau$. 
Using this fact,
we derive comparable bounds on the Rademacher complexity of the two function classes by splitting $f$ in terms of how the observed features $\bx$ align with the mean difference vector $\Delta$.
%Here, we let $\bx_\perp := \Pi \bx$ be the component of $\bx$ that is parallel to the mean discrepancy i.e., parallel to the ``irrelevant component,'' and hence perpendicular to the relevant components. We 
Analogously to $\bx_\perp$ above, let $\bx_\parallel := (I-\Pi)\bx$ be the orthogonal component to the mean discrepancy vector $\Delta$.
%(i.e., the ``relevant component'').

\begin{thmprop}~\label{prop:rad_fc_fm}
Let $\bx_\perp := \Pi \bx$, $\bx_\parallel := (I-\Pi)\bx$.
For training data $\cD = \{(\bx_i, y_i, v_i)\}_{i=1}^{n}$, $\cD \sim P^\circ$,  
$\sup_{\bx_\perp} \|\bx_\perp\|_2 \leq B_\perp$,
$\sup_{\bx_\parallel} \|\bx_\parallel \|_2 \leq B_\parallel$, then
% \noindent\begin{subequations}
% \vspace{-0.5em}
% \begin{minipage}{.4\columnwidth}
%     \begin{align*} 
%         \mathfrak{R}(\cF_{L_2}) \leq \frac{A \sqrt{B_\parallel^2 + B_\perp^2}}{\sqrt{n}},
%     \end{align*}
% \end{minipage}%
% \begin{minipage}{.4\columnwidth}
%     \begin{align*} 
%       \mathfrak{R}(\cF_{\mmd,L_2}) \leq \frac{A \cdot B_\parallel + \tau \frac{B_\perp}{\|\Delta\|}}{\sqrt{n}}.
%     \end{align*}
% \end{minipage}
% \vspace{-0.5em}
% \end{subequations}
\vspace{-0.5em}
\begin{gather*} 
\mathfrak{R}(\cF_{L_2}) \leq \frac{A \sqrt{B_\parallel^2 + B_\perp^2}}{\sqrt{n}},\\
\mathfrak{R}(\cF_{L_2, \mmd}) \leq \frac{A \cdot B_\parallel + \tau \frac{B_\perp}{\|\Delta\|}}{\sqrt{n}}.
\end{gather*}
\end{thmprop}
\vspace{-1em}
The proof applies a standard Rademacher complexity bound for the $L_2$ class, and applies a more conservative strategy for $\cF_{L_2, \mmd}$ by separately bounding the worst-case terms involving $\bx_\perp$ and $\bx_\parallel$.
Nonetheless, comparing these bounds is instructive.
The upper bound on $\mathfrak{R}(\cF_{\mmd,L_2})$ is smaller than that of $\mathfrak{R}(\cF_{L_2})$ whenever $\tau$ satisfies: 
\vspace{-0.5em}
\begin{align}\label{eq:tau_condition}
0 \leq \tau < A\left[\sqrt{B_\parallel^2+B_\perp^2} - B_\parallel\right]\frac{\|\Delta\|}{B_\perp}.
\end{align}
The ratio $\|\Delta\| / B_\perp$ characterizes how much of the variation in $\bx_{\perp}$ comes from the mean shift in $\bx$ induced by $v$.
When the variation from this shift is large relative to variation in the $\bx_\parallel$ directions, we expect even weak $\mmd$ regularization to yield better generalization than $L_2$ regularization alone.
This occurs in cases where $V$ controls features that are highly salient in the input $\bx$.
For example, in image classification, if $V$ denotes the background type, we expect $\|\Delta\| / B_\perp$ to be large if the background features are very different between $V=0$ and $V=1$, but relatively consistent within values of $V$.
Further, if the background accounts for the majority of pixels in each image, we expect $B_\perp \gg B_\parallel$, resulting in an even stronger regularizing effect from the $\mmd$ penalty.

\textbf{Why is weighting important?}
When $P_s \not=P^\circ$, enforcing the invariance penalty without reweighting leads to a model that is inconsistent with the causal DAG, and is hence biased (i.e., inaccurate). This means that typical invariance penalties (e.g., \cite{krueger2021out, donini2018empirical}) will lead to an invariance-accuracy tradeoff. We show this formally in the appendix, highlighting why enforcing the $\mmd$ penalty without reweighting when $P_s \not=P^\circ$ can introduce biased estimation. These findings have implications that extend beyond shortcut learning, e.g., fairness and causality where the $\mmd$ penalty is frequently used.

\section{EXPERIMENTS}\label{sec:experiments}
We now empirically demonstrate that our approach mitigates shortcut learning by showing that predictors learned with our approach are robust to distribution shifts that invalidate the shortcut.
%We show that our approach leads to efficient predictors that are robust to distribution shift. 
% We also show that, as our theory suggests, our approach leads to more efficient estimation under the ``ideal'' distribution even in the absence of correlated sampling or distribution shift. 
We study two tasks: predicting bird types from images, and detecting pneumonia using chest X-rays\footnote{Our code is available at \href{https://github.com/mymakar/causally_motivated_shortcut_removal}{\texttt{github.com/mymakar/causally\_motivated\_shortcut\_removal}}}.

\vspace{-1em}
\subsection{Water birds}\label{sec:waterbirds}
\vspace{-1em}
We test our method on a semi-synthetic task adapted from \cite{sagawa2019distributionally}.
The goal is to predict the type of bird ($Y$) appearing in an image ($\bX$) considering the background ($V$) to be a possible shortcut.
%Following \cite{sagawa2019distributionally}, we study using images $(\bX)$, to predict if the bird in the image is a water or land bird ($Y$), considering the background (land or water) of the image to be a possible shortcut $(V)$. 

\textbf{Data.} We construct a dataset that combines images of water birds (Gulls) and land birds (Warblers) extracted from the Caltech-UCSD Birds-200-2011 (CUB) dataset \citep{wah2011caltech} with water and land background extracted from the Places dataset \citep{zhou2017places}. Additional details about data generation, hyperparameter selection, and code are included in the supplementary materials. 
\begin{comment}
We found that the original background images frequently contain landscapes that are difficult to distinguish (e.g., water backgrounds with very small water bodies that mostly reflect the surrounding trees). Instead, we pick 300 ``clean'' images for each of the land and water backgrounds. Using those clean images, we generate 10,000 land backgrounds, and 9,000 water backgrounds by applying random transformations (rotation, zoom, darkening/brightening) to the selected images. 
In the appendix, we present the results on the original backgrounds.
\end{comment}

We examine how a model trained under a specific distribution $P_s$ generalizes to various test distributions $P_t \in \cP$.
We consider two different cases for $P_s$.
In the first, we generate the training data from the ideal distribution $P^\circ$, with $P^\circ(Y|V=1) = P^\circ(Y|V=0) = 0.5$.
This setting tests the implications of proposition~\ref{prop:rad_fc_fm}, which suggests that our approach improves sample efficiency even under uncorrelated sampling.
% This is one of the main mechanisms by which our approach is meant to reduce the risk gap in $\eqref{eq:overall_goal}$.
In the second setting, the training data is sampled from a $P_s$ where the auxiliary label and the main label are correlated, i.e., $V \not \indep Y$. Specifically, we generate the data such that $P_s(Y=1 | V=1) = P_s(Y=0 | V=0) = 0.9$. 
%For the first setting, we generate the training data from the ideal distribution $P^\circ$, with $P^\circ(Y|V=1) = P^\circ(Y|V=0) = 0.5$. In the second setting, we generate the data such that $P(Y=1 | V=1) = P(Y=0 | V=0) = 0.9$, representing a scenario where the majority of water birds are on water backgrounds and the majority of land birds are on land backgrounds. 
For both settings, we introduce noise by randomly flipping 1\% of each of the labels. We generate a number of held-out test sets, each one corresponding to a different probability of observing a waterbird with a water background, and similarly with land birds. 

We use ResNet-50 \citep{he2016deep}, pretrained on ImageNet, and fine tuned for our task. All models in this paper are implemented in TensorFlow \citep{tensorflow2015-whitepaper}. We present the results from 20 simulations. In each simulation, we generate different train/test splits, and different bird-background combinations. 

\textbf{Baselines.}
We compare the following methods in our experiments:
\begin{enumerate*}[label={(\arabic*)}]
% \begin{enumerate}
   \item \textbf{wMMD-T} is our main proposal, which minimizes \eqref{eqn:empirical_mmd_loss}, and applies the two-step cross validation process described in Section~\ref{sec:summary_approach}.
    \item \textbf{cMMD-T} is similar to our main approach but does not incorporate the weights $u$, and penalizes an $\mmd$ penalty that is conditional on the main label $Y$. Enforcing the conditional independence is consistent with the causal DAG so we expect this baseline to give unbiased estimators as is shown in \cite{veitch2021counterfactual}. 
    \item \textbf{L2-S} is the standard DNN trained to minimize the empirical risk, with an L2 penalty on the weights. 
    % We introduce regularization by penalizing the L2-norm of the weights, picking the value of the penalty from 0.0 (no regularization), 0.0001, and 0.001 which is similar to values typically used for this setting \citep{sagawa2019distributionally,he2016deep}. 
    \item \textbf{wL2-S}: similar to L2-S but also incorporates weighting using $\bu$ as defined in \eqref{eq:weights}. 
    \item \textbf{Rand-Aug-S}: a baseline that aims for robustness by augmenting the data at training time using  random flips and rotations.  
    \item \textbf{Rex}: a recently suggested method that imposes an invariance penalty without re-weighting \citep{krueger2021out}. We show the results from Rex in the appendix since it significantly under performs compared to other models.  
\end{enumerate*}

For L2-S, wL2-S, Rand-Aug-S, and Rex cross-validation is done in the standard way by picking the model that has the best accuracy on the validation set.
 
In the uncorrelated setting, we only present the unweighted variants of all the models, since the weights are roughly constant across data points in that setting. 

\textbf{Ablations.}
We study ablated variants of our approach:
\begin{enumerate*}[label={(\arabic*)}]
    \item \textbf{wMMD-S} is similar to wMMD-T,
    %this model minimizes equation~\ref{eqn:empirical_mmd_loss}
    but uses standard cross-validation rather than our two-step cross-validation. 
    \item \textbf{MMD-T} minimizes a variant of equation~\ref{eqn:empirical_mmd_loss} that does not utilize $\bu$ in the logisitc loss or the $\mmd$ penalty. At validation time hyperparameters are picked using our two-step cross-validation approach using $\bu-$weighted metrics. 
    \item \textbf{MMD-S} is similar to MMD-T, but does standard cross-validation. 
    \item \textbf{MMD-uT} is similar to MMD-T, but uses unweighted validation metrics in our two-step cross-validation. 
\end{enumerate*}

\begin{figure*}[htp]
\vspace{-1mm}
\centering
\vspace{-1em}
\includegraphics[width=\textwidth]{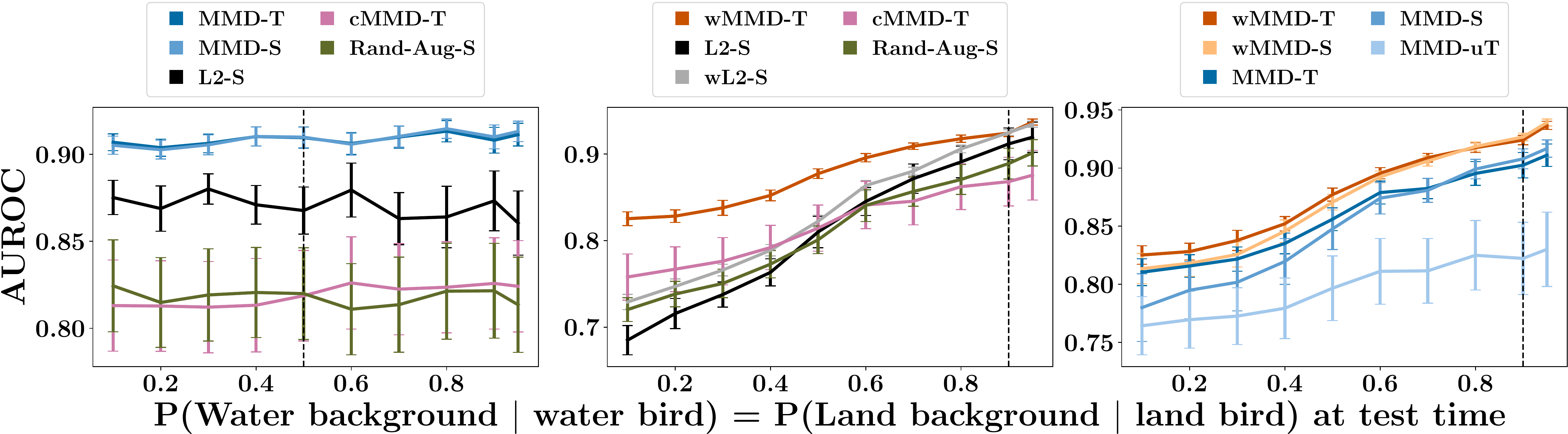}
\caption{
\textbf{In all plots} $x$-axis shows $P(Y|V)$ at test time under different shifted distributions, $y$-axis shows AUROC on test data, and vertical dashed line shows $P(Y|V)$ at training time.
\textbf{(Left)} Training data sampled from the ideal (uncorrelated) distribution. MMD-regularized models outperform baselines within, and outside the training distribution.
\textbf{(Middle)} Training data sampled from a correlated distribution.  MMD-regularized model outperform baselines showing better robustness against distribution shifts at test time.
\textbf{(Right)} Ablation study shows how different components of our suggested approach (wMMD-T) contribute to improved performance. \label{fig:main} \vspace{-1em}}
\end{figure*}

\textbf{Results: Sampling from the ideal distribution}.
Figure~\ref{fig:main} (left) shows the results from the first setting, where the training data is sampled from the ideal, uncorrelated distribution $P^\circ$. 
The $x-$axis shows $P(Y=1|V=1) = P(Y=0|V=0)$ at test time, while the $y-$axis shows the corresponding mean AUROC, averaged over 20 simulations. The vertical dashed line shows the conditional probability at training time. Both variants of our proposed approach, with standard and two-step cross-validation outperform other baselines within the training distribution and when there is distribution shift. This conforms with our theory that suggests that even when the data are sampled from the ideal distribution, using a causally-motivated regularization scheme leads to more efficient models, which translates into better performance in finite samples. 
Despite being consistent with the causal DAG, cMMD-T requires conditioning on $Y$ to estimate the $\mmd$ penalty.
We conjecture that the need to slice the data on $Y$ leads to unstable estimates of the conditional $\mmd$ when using standard batch sizes. 
%This leads to data slicing, and effectively smaller sample size which ultimately leads to poor performance due to unstable training dynamics in small batch size.
Additional experimental results that support this conjecture are included in the appendix.
In addition, all models are robust to distribution shift when trained on data from $P^\circ$. This conforms with proposition~\ref{prop:frob_ferm_equal}, which states that under $P^\circ$, optimal predictors are risk-invariant. 

\textbf{Results: Sampling from a correlated distribution}. 
Figure~\ref{fig:main} (middle) shows the results from the setting where the training data is sampled a correlated distribution with $P_s(Y=1|V=1) = P_s(Y=0|V=0) = 0.9$. The $x$, and $y$ axes are similar to figure~\ref{fig:main} (left). Our approach (wMMD-T) outperforms other models especially at high divergence from the training distribution. Out of all the non-MMD regularized baselines, the weighted L2-regularized model performs best. This suggests that minimizing the empirical risk on the $\bu-$reweighted distribution contributes to model robustness. Conclusions about cMMD-T are similar to the uncorrelated distribution setting. Studying performance metrics other than the AUROC (e.g., Brier score or logistic loss) yields the same conclusions (see the appendix section \ref{sec:experiments_extra}). 

The ablation study in Figure~\ref{fig:main} (right) shows that the largest increase in performance is attributable to
%weighting by $\bu$
using the weighted $\mmd$ penalty 
at training time, since the two weighted $\mmd$ variants outperform the two unweighted variants. Within those two groups, the two-step cross validation with weighted metrics outperforms the others, especially in terms of robustness. This shows that when training models using the $\mmd$ penalty, it is important to take into consideration that, unlike L2-norm regularization, the penalty depends on the training data, and is prone to overfitting. 
%The results also show that it is possible to improve the performance of models which are unweighted at training time by using our two step cross validation approach with weighted validation metrics, since MMD-T slightly outperforms MMD-S.
Recall that MMD-uT strictly enforces the $\mmd$ penalty without addressing the fact that the training distribution has been sampled from a correlated distribution.
This yields a fairly robust predictor that has poor performance.
This conforms with our findings stated in the appendix proposition~\ref{prop:looser_mmd}, which imply that there will be a bias-robustness trade-off if the correlated sampling is not corrected. 

\vspace{-1em}
\subsection{Chest X-rays}
\vspace{-1em}
For a less synthetic test, we adapt an experiment from \cite{jabbour2020deep} where the task of predicting pneumonia ($Y=1$) vs. no clinical findings ($Y=0$) from chest X-rays ($\bX$) considering sex to be a possible shortcut ($V)$. 

\textbf{Data.} We conduct this analysis on a publicly available dataset, CheXpert \citep{irvin2019chexpert}. 
%The data contain chest x-rays of patients collected from the Stanford Hospital over 15 years. Each chest x-ray is associated sex of the patient as well as Pneumonia status. 
At training time, we under-sample women who did not have pneumonia to create setting where sex-related attributes are possible shortcuts to predict pneumonia. Specifically, we sample the data such that $P_s(Y=1) = 0.3$, and the majority of women have pneumonia whereas the majority of men do not have pneumonia i.e.,  $P(Y=1 | V=1) = P(Y=0 | V=0) = 0.9$. 
For this task, we use DensNet-121 \citep{huang2017densely}, pretrained on ImageNet, and fine tuned for our specific task. We use DenseNet because it was shown to outperform other commonly used architectures on the CheXpert dataset \citep{irvin2019chexpert, ke2021chextransfer,jabbour2020deep}. We present the results from 5 simulations. In each simulation, we generate different train/test splits. Additional details about the training are presented in the appendix.
% We use a batch size = 16 following similar work \cite{irvin2019chexpert, ke2021chextransfer,jabbour2020deep}.

\textbf{Results.}
We evaluate the performance of our approach and baselines similar to those outlined in section~\ref{sec:waterbirds} in the \textbf{no shift} setting where the test data is sampled from the same biased distribution as the training data, and a \textbf{shift} setting where the test data comes from a distribution where $P(Y=1 | V=1) = P(Y=0 | V=0) = 0.5$. The results in table~\ref{tab:chexpert} show that our proposed approach (wMMD-T) outperforms all others when there is a distribution shift and performs comparably in-distribution to the L2-regularized DNN that utilizes our proposed weighting scheme (wL2-S). The performance of MMD-uT, cMMD-T and Rex does not change as much as other models across the two distributions, signaling robustness. However, MMD-uT and Rex under-perform in terms of accuracy, which is expected since they enforce an invariance penalty that is inconsistent with the causal DAG and are subject to a invariance-accuracy trade-off. Meanwhile, cMMD-T is consistent with the causal DAG, but  under-performs in practice. This is consistent with our conjecture that the cMMD-T penalty is harder to estimate using small batches.

\begin{table}[]
    \centering
    %\resizebox{0.7\columnwidth}{!}{%
        \begin{tabular}{l|ll}
        \toprule
            & \multicolumn{2}{c}{AUROC (STE)}\\
        \midrule
              Model & Shift & No shift \\
        \midrule
        wMMD-T &    \textbf{0.75 (0.006)} &       0.85 (0.007) \\ %
        MMD-uT &     0.7 (0.013) &       0.75 (0.033) \\
        cMMD-T &    0.71 (0.024) &        0.73 (0.05) \\
        L2-S &    0.62 (0.066) &       0.71 (0.063) \\ %
         wL2-S &    0.69 (0.019) &        \textbf{0.86 (0.01)} \\
        Rex &    0.57 (0.036) &        0.6 (0.038) \\
        \bottomrule
        \end{tabular}
        %}
    \caption{CheXpert results show that our approach (wMMD-T) outperforms others when there is a distribution shift and performs comparably to the L2-regularized DNN that utilizes our proposed weighting scheme (wL2-S) in distribution.}
    \label{tab:chexpert}
    \vspace{-1em}
\end{table}

% \begin{figure}[htp]
% \centering
% \includegraphics[width=0.4\textwidth]{figures/ablation_8090.pdf}
% \captionof{figure}{Training data sampled from $P$, with $P(Y=1|V=1) = P^\circ(Y=0|V=0) = 0.9$. $x$, $y$ axes similar to figure~\ref{fig:main_a}.An ablation study to show how different components of our suggested approach (wMMD-T) contribute to improved performance. }
% \label{fig:ablation}
% \end{figure}

\section{CONNECTIONS TO EXISTING WORK}
\label{sec:literature review}

% MM: Papers we talked about that didnt make the cut 
% Fair PCA https://arxiv.org/abs/1811.00103
% semi-supervised setting https://arxiv.org/pdf/2006.10032.pdf

% https://arxiv.org/abs/1811.00103
% \begin{enumerate}
%     \item Underspecification issue has been discussed before (cite credibility revolution, shortcuts paper, predictive multiplicity, rashomon sets)
%     \item Fairness + use of MMD's
%     \item Causal stuff (like us, but independence is enforced for one variable and the objective is not regularization but rather learning causal effects) 
%     \item Invariance stuff (IRM and learning in different environments). There, ``environment'' is the relevant metadata. Uri's HSIC with residual's paper. Moyer et al https://arxiv.org/abs/1805.09458.
%     \item Augmentation literature (like our setup but they assume that the correct augmentations are known, we make no such assumption) 
%     \item Disentanglement literature (Not robust to independence?)
% \end{enumerate}

% Our work unifies several threads that have appeared in the ML literature.

\textbf{Shortcut learning.} Discouraging shortcut learning using data augmentation (e.g., rotation, translation, cropping) has been suggested by multiple authors \citep{hendrycks2020many,yin2019fourier,lyle2020benefits,lopes2019improving,cubuk2018autoaugment}. 
This approach can work if a generator for shortcut transformations is available at training time.
However, if this set of generated transformations does not include interventions that disrupt the shortcut, the desired robustness might not be achieved, as evidenced by the empirical performance of the random augmentation baseline presented in the experiments section.
%By contrast, we do not claim to know these transformations. 
Our approach does not require generating shortcut transformations.
Instead, our approach leads to invariant models by leveraging the auxiliary labels to inform the relevant transformations the main label should be independent to.
Other work views shortcuts as a consequence of model overparameterization.
For example, \cite{sagawa2020investigation} observe that overparameterization exacerbates the reliance on spurious correlations, and suggest an approach similar to wL2-S, which is outperformed by our model. 
% Distributionally robust optimization has also been proposed as a remedy for shortcut learning~\citep{sagawa2019distributionally}. 

\textbf{Fairness and invariance}. Our work sheds light on properties of invariant representations, which are used in the fairness literature \citep{madras2018learning,donini2018empirical}. 
One key distinction between our work and fairness literature is that our invariance penalty is designed to be consistent with the causal DAG.
This leads to estimators that are asymptotically optimal.
By contrast, in the the fairness literature, invariance constraints are often motivated by external criteria, and may be incompatible with the causal DAG.
In these cases, a tradeoff can arise between the fairness criterion and accuracy (see proposition~\ref{prop:looser_mmd} in the appendix).
Thus our work contributes to ongoing investigations about whether there is a trade-off between fariness and accuracy \citep{zhang2019theoretically,calders2009building,johansson2019support,dutta2020there,zhao2019inherent}.
Note that the MMD-uT and the cMMD-T baselines presented in the experiments section correspond to penalties that enforce demographic parity and equalized odds respectively \citep{hardt2016equality}. 

Our theoretical analysis is related to analysis of generalization in ``Fair ERM'' appearing in \citet{donini2018empirical}, but has a different focus.
Specifically, our analysis (Proposition~\ref{prop:wdelta_bound}) highlights that an invariance constraint can \emph{improve} efficiency by reducing the complexity of a learning problem through regularization of ``orthogonal'' dimensions.
%, showing that the complexity of the function class of the invariant functions decreases because of the imposition of the invariance penalty. 

Similar invariance penalties have been suggested in causality literature \citep{shalit2017estimating,johansson2016learning}, and domain shift literature \citep{tzeng2014deep,long2015learning}. While proposition~\ref{prop:risk_gap} bears some similarity to statements presented in the domain adaptation literature \citep{long2015learning, ben2007analysis,ben2010theory}, our work is distinct in that we do not aim to generalize to a specific target domain. Instead, we aim to build models that generalize across a \emph{family} of target domains. One consequence of this is that, unlike unsupervised domain adaptation, we do not require access to examples from a target domain. 

\textbf{Causally-motivated invariance.} 
%Our work is similar to anchor regression \citep{rothenhausler2018anchor} in that we also view invariance through the lens of causality. By contrast, we do not assume linear relationships between $\bX, Y,$ and the ``anchor'' variable $V$, and we are not limited to linear models. 
In other work, 
\citet{arjovsky2019invariant} propose an invariant risk minimization (IRM) approach inspired by ideas from causality. Unlike our approach, IRM does not explicitly penalize dependence on the redundant dimensions, but instead relies on the idea that the invariant risk minimizer should achieve the lowest error across datasets sampled from different target distributions $P_t$. As others (e.g.,\cite{guo2021out, rosenfeld2020risks}) noted, when the family of functions is as flexible as DNNs, it is possible to find a predictor that achieves the objective of IRM but is not robust. 
Similar to us, \citet{krueger2021out} suggest a model for risk extrapolation (Rex) in the anti-causal settings. Their method does not correct for biased sampling and hence has the same limitations as the unweighted $\mmd$ models presented in the experiments. Rex results presented in the appendix confirm that it consistently performs worse than our approach.  

Similar to this work, \citet{subbaswamy2018counterfactual} develop methods to create estimators that are stable across distribution shifts by appealing to causal DAGS and without requiring access to samples from the target distribution. However, they do so by identifying stable features or ``conditioning sets'' that only contain variables with stable relationships with the target label. Such an approach, which assumes that the input features $\bX$ are interpretable, would not be appropriate in a more general setting where the input features are image pixels. 

Our work is most similar and complementary to \citet{veitch2021counterfactual}, where the authors also note that the causal structure of a problem has implications for distributional robustness. Unlike \citet{veitch2021counterfactual}, we focus on the specific anti-causal setting described in figure~\ref{fig:dags}, and provide a finite sample analysis highlighting the robustness and efficiency of our estimator. We note that our results extend to non anti-causal DAGs, so long as the correlation between $Y$, and $V$ is purely spurious, which is consistent with the findings of \citet{veitch2021counterfactual}.

\section{DISCUSSION\label{sec:conc}}
We presented an approach to build models that are invariant to distribution shifts using auxiliary labels. 
Guided by our theoretical insight, we suggested a causally-motivated regularization scheme to train robust, and accurate models. We showed that our approach empirically outperforms others. 

\textbf{Limitations.} Our approach requires \emph{a priori} knowledge of a shortcut that might be exploited by a predictor.
This is to be expected: the choice of which source of variation to exclude from a predictor is highly context-specific and problem-specific.
If a predictor exhibit a lack of robustness and shortcut learning is suspected, practitioners could apply exploratory techniques (e.g., interpretability method such as saliency maps \cite{simonyan2013deep} or Shapley values \cite{lundberg2017unified,wang2021shapley}) to surface factors that the model may rely on.
Ultimately, however, this requires judgment about the problem domain.
%If the interpretability analysis reveals that the model is relying on a shortcut, and if an auxiliary label that corresponds to that shortcut is available, our proposed approach can then be used. 
A second limitation is the focus on one binary auxiliary variable. For multiple/categorical $V$ our approach could be modified in several ways. First, for a categorical $V$ with $m$ categories, the $\mmd$ term in equation~\eqref{eq:overall_goal} can be replaced with $\sfrac{m !}{2 (m-2)!}$ $\mmd$ terms, each corresponding to a pairwise comparison between two groups defined by $V$. 
Alternatively, we would define the invariance penalty with respect to the Hilbert Schmidt independence criterion (HSIC), which allows for higher cardinality \citep{gretton2007kernel}.
%Second, instead of using inverse weights similar to the ones defined in equation~\eqref{eq:weights} which assume that $V$ is binary, permutation weights \citep{arbour2021permutation} which allow for arbitrary cardinality of $V$ can be used. 

\subsubsection*{Acknowledgements}
We thank the reviewers for their insightful comments. We thank John Guttag and the members of the Clinical and Applied Machine Learning group at MIT for their detailed feedback. Part of this work was done while MM was an intern at Google Research. 

\bibliographystyle{unsrtnat}
\bibliography{slabs}

\begin{thebibliography}{57}
\providecommand{\natexlab}[1]{#1}
\providecommand{\url}[1]{\texttt{#1}}
\expandafter\ifx\csname urlstyle\endcsname\relax
  \providecommand{\doi}[1]{doi: #1}\else
  \providecommand{\doi}{doi: \begingroup \urlstyle{rm}\Url}\fi

\bibitem[Beery et~al.(2018)Beery, Van~Horn, and Perona]{beery2018recognition}
Sara Beery, Grant Van~Horn, and Pietro Perona.
\newblock Recognition in terra incognita.
\newblock In \emph{Proceedings of the European Conference on Computer Vision
  (ECCV)}, pages 456--473, 2018.

\bibitem[Ilyas et~al.(2019)Ilyas, Santurkar, Tsipras, Engstrom, Tran, and
  Madry]{ilyas2019adversarial}
Andrew Ilyas, Shibani Santurkar, Dimitris Tsipras, Logan Engstrom, Brandon
  Tran, and Aleksander Madry.
\newblock Adversarial examples are not bugs, they are features.
\newblock In \emph{Advances in Neural Information Processing Systems}, pages
  125--136, 2019.

\bibitem[Azulay and Weiss(2018)]{azulay2018deep}
Aharon Azulay and Yair Weiss.
\newblock Why do deep convolutional networks generalize so poorly to small
  image transformations?
\newblock \emph{arXiv preprint arXiv:1805.12177}, 2018.

\bibitem[Geirhos et~al.(2018)Geirhos, Temme, Rauber, Sch{\"u}tt, Bethge, and
  Wichmann]{geirhos2018generalisation}
Robert Geirhos, Carlos~RM Temme, Jonas Rauber, Heiko~H Sch{\"u}tt, Matthias
  Bethge, and Felix~A Wichmann.
\newblock Generalisation in humans and deep neural networks.
\newblock In \emph{Advances in neural information processing systems}, pages
  7538--7550, 2018.

\bibitem[Taori et~al.(2020)Taori, Dave, Shankar, Carlini, Recht, and
  Schmidt]{taori2020measuring}
Rohan Taori, Achal Dave, Vaishaal Shankar, Nicholas Carlini, Benjamin Recht,
  and Ludwig Schmidt.
\newblock Measuring robustness to natural distribution shifts in image
  classification.
\newblock \emph{Advances in Neural Information Processing Systems}, 33, 2020.

\bibitem[Geirhos et~al.(2020)Geirhos, Jacobsen, Michaelis, Zemel, Brendel,
  Bethge, and Wichmann]{geirhos2020shortcut}
Robert Geirhos, J{\"o}rn-Henrik Jacobsen, Claudio Michaelis, Richard Zemel,
  Wieland Brendel, Matthias Bethge, and Felix~A Wichmann.
\newblock Shortcut learning in deep neural networks.
\newblock \emph{arXiv preprint arXiv:2004.07780}, 2020.

\bibitem[Sagawa et~al.(2019)Sagawa, Koh, Hashimoto, and
  Liang]{sagawa2019distributionally}
Shiori Sagawa, Pang~Wei Koh, Tatsunori~B Hashimoto, and Percy Liang.
\newblock Distributionally robust neural networks for group shifts: On the
  importance of regularization for worst-case generalization.
\newblock \emph{arXiv preprint arXiv:1911.08731}, 2019.

\bibitem[Nagarajan et~al.(2020)Nagarajan, Andreassen, and
  Neyshabur]{nagarajan2020understanding}
Vaishnavh Nagarajan, Anders Andreassen, and Behnam Neyshabur.
\newblock Understanding the failure modes of out-of-distribution
  generalization.
\newblock \emph{arXiv preprint arXiv:2010.15775}, 2020.

\bibitem[Sagawa et~al.(2020{\natexlab{a}})Sagawa, Raghunathan, Koh, and
  Liang]{pmlr-v119-sagawa20a}
Shiori Sagawa, Aditi Raghunathan, Pang~Wei Koh, and Percy Liang.
\newblock An investigation of why overparameterization exacerbates spurious
  correlations.
\newblock In Hal~Daumé III and Aarti Singh, editors, \emph{Proceedings of the
  37th International Conference on Machine Learning}, volume 119 of
  \emph{Proceedings of Machine Learning Research}, pages 8346--8356. PMLR,
  13--18 Jul 2020{\natexlab{a}}.

\bibitem[D'Amour et~al.(2020)D'Amour, Heller, Moldovan, Adlam, Alipanahi,
  Beutel, Chen, Deaton, Eisenstein, Hoffman,
  et~al.]{damour2020underspecification}
Alexander D'Amour, Katherine Heller, Dan Moldovan, Ben Adlam, Babak Alipanahi,
  Alex Beutel, Christina Chen, Jonathan Deaton, Jacob Eisenstein, Matthew~D
  Hoffman, et~al.
\newblock Underspecification presents challenges for credibility in modern
  machine learning.
\newblock \emph{arXiv preprint arXiv:2011.03395}, 2020.

\bibitem[Locatello et~al.(2019)Locatello, Bauer, Lucic, Raetsch, Gelly,
  Sch{\"o}lkopf, and Bachem]{locatello2019challenging}
Francesco Locatello, Stefan Bauer, Mario Lucic, Gunnar Raetsch, Sylvain Gelly,
  Bernhard Sch{\"o}lkopf, and Olivier Bachem.
\newblock Challenging common assumptions in the unsupervised learning of
  disentangled representations.
\newblock In \emph{international conference on machine learning}, pages
  4114--4124, 2019.

\bibitem[Gretton et~al.(2012)Gretton, Borgwardt, Rasch, Sch{\"o}lkopf, and
  Smola]{gretton2012kernel}
Arthur Gretton, Karsten~M Borgwardt, Malte~J Rasch, Bernhard Sch{\"o}lkopf, and
  Alexander Smola.
\newblock A kernel two-sample test.
\newblock \emph{The Journal of Machine Learning Research}, 13\penalty0
  (1):\penalty0 723--773, 2012.

\bibitem[Golowich et~al.(2018)Golowich, Rakhlin, and Shamir]{golowich2018size}
Noah Golowich, Alexander Rakhlin, and Ohad Shamir.
\newblock Size-independent sample complexity of neural networks.
\newblock In \emph{Conference On Learning Theory}, pages 297--299. PMLR, 2018.

\bibitem[Mohri et~al.(2018)Mohri, Rostamizadeh, and
  Talwalkar]{mohri2018foundations}
Mehryar Mohri, Afshin Rostamizadeh, and Ameet Talwalkar.
\newblock \emph{Foundations of machine learning}.
\newblock MIT press, 2018.

\bibitem[Cortes et~al.(2010)Cortes, Mansour, and Mohri]{cortes2010learning}
Corinna Cortes, Yishay Mansour, and Mehryar Mohri.
\newblock Learning bounds for importance weighting.
\newblock In \emph{Nips}, volume~10, pages 442--450. Citeseer, 2010.

\bibitem[Krueger et~al.(2021)Krueger, Caballero, Jacobsen, Zhang, Binas, Zhang,
  Le~Priol, and Courville]{krueger2021out}
David Krueger, Ethan Caballero, Joern-Henrik Jacobsen, Amy Zhang, Jonathan
  Binas, Dinghuai Zhang, Remi Le~Priol, and Aaron Courville.
\newblock Out-of-distribution generalization via risk extrapolation (rex).
\newblock In \emph{International Conference on Machine Learning}, pages
  5815--5826. PMLR, 2021.

\bibitem[Donini et~al.(2018)Donini, Oneto, Ben-David, Shawe-Taylor, and
  Pontil]{donini2018empirical}
Michele Donini, Luca Oneto, Shai Ben-David, John Shawe-Taylor, and Massimiliano
  Pontil.
\newblock Empirical risk minimization under fairness constraints.
\newblock In \emph{Proceedings of the 32nd International Conference on Neural
  Information Processing Systems}, pages 2796--2806, 2018.

\bibitem[Wah et~al.(2011)Wah, Branson, Welinder, Perona, and
  Belongie]{wah2011caltech}
Catherine Wah, Steve Branson, Peter Welinder, Pietro Perona, and Serge
  Belongie.
\newblock The caltech-ucsd birds-200-2011 dataset.
\newblock 2011.

\bibitem[Zhou et~al.(2017)Zhou, Lapedriza, Khosla, Oliva, and
  Torralba]{zhou2017places}
Bolei Zhou, Agata Lapedriza, Aditya Khosla, Aude Oliva, and Antonio Torralba.
\newblock Places: A 10 million image database for scene recognition.
\newblock \emph{IEEE transactions on pattern analysis and machine
  intelligence}, 40\penalty0 (6):\penalty0 1452--1464, 2017.

\bibitem[He et~al.(2016)He, Zhang, Ren, and Sun]{he2016deep}
Kaiming He, Xiangyu Zhang, Shaoqing Ren, and Jian Sun.
\newblock Deep residual learning for image recognition.
\newblock In \emph{Proceedings of the IEEE conference on computer vision and
  pattern recognition}, pages 770--778, 2016.

\bibitem[Abadi et~al.(2015)Abadi, Agarwal, Barham, Brevdo, Chen, Citro,
  Corrado, Davis, Dean, Devin, Ghemawat, Goodfellow, Harp, Irving, Isard, Jia,
  Jozefowicz, Kaiser, Kudlur, Levenberg, Man\'{e}, Monga, Moore, Murray, Olah,
  Schuster, Shlens, Steiner, Sutskever, Talwar, Tucker, Vanhoucke, Vasudevan,
  Vi\'{e}gas, Vinyals, Warden, Wattenberg, Wicke, Yu, and
  Zheng]{tensorflow2015-whitepaper}
Mart\'{\i}n Abadi, Ashish Agarwal, Paul Barham, Eugene Brevdo, Zhifeng Chen,
  Craig Citro, Greg~S. Corrado, Andy Davis, Jeffrey Dean, Matthieu Devin,
  Sanjay Ghemawat, Ian Goodfellow, Andrew Harp, Geoffrey Irving, Michael Isard,
  Yangqing Jia, Rafal Jozefowicz, Lukasz Kaiser, Manjunath Kudlur, Josh
  Levenberg, Dandelion Man\'{e}, Rajat Monga, Sherry Moore, Derek Murray, Chris
  Olah, Mike Schuster, Jonathon Shlens, Benoit Steiner, Ilya Sutskever, Kunal
  Talwar, Paul Tucker, Vincent Vanhoucke, Vijay Vasudevan, Fernanda Vi\'{e}gas,
  Oriol Vinyals, Pete Warden, Martin Wattenberg, Martin Wicke, Yuan Yu, and
  Xiaoqiang Zheng.
\newblock {TensorFlow}: Large-scale machine learning on heterogeneous systems,
  2015.
\newblock URL \url{https://www.tensorflow.org/}.
\newblock Software available from tensorflow.org.

\bibitem[Veitch et~al.(2021)Veitch, D'Amour, Yadlowsky, and
  Eisenstein]{veitch2021counterfactual}
Victor Veitch, Alexander D'Amour, Steve Yadlowsky, and Jacob Eisenstein.
\newblock Counterfactual invariance to spurious correlations: Why and how to
  pass stress tests.
\newblock \emph{arXiv preprint arXiv:2106.00545}, 2021.

\bibitem[Jabbour et~al.(2020)Jabbour, Fouhey, Kazerooni, Sjoding, and
  Wiens]{jabbour2020deep}
Sarah Jabbour, David Fouhey, Ella Kazerooni, Michael~W Sjoding, and Jenna
  Wiens.
\newblock Deep learning applied to chest x-rays: Exploiting and preventing
  shortcuts.
\newblock In \emph{Machine Learning for Healthcare Conference}, pages 750--782.
  PMLR, 2020.

\bibitem[Irvin et~al.(2019)Irvin, Rajpurkar, Ko, Yu, Ciurea-Ilcus, Chute,
  Marklund, Haghgoo, Ball, Shpanskaya, et~al.]{irvin2019chexpert}
Jeremy Irvin, Pranav Rajpurkar, Michael Ko, Yifan Yu, Silviana Ciurea-Ilcus,
  Chris Chute, Henrik Marklund, Behzad Haghgoo, Robyn Ball, Katie Shpanskaya,
  et~al.
\newblock Chexpert: A large chest radiograph dataset with uncertainty labels
  and expert comparison.
\newblock In \emph{Proceedings of the AAAI conference on artificial
  intelligence}, volume~33, pages 590--597, 2019.

\bibitem[Huang et~al.(2017)Huang, Liu, Van Der~Maaten, and
  Weinberger]{huang2017densely}
Gao Huang, Zhuang Liu, Laurens Van Der~Maaten, and Kilian~Q Weinberger.
\newblock Densely connected convolutional networks.
\newblock In \emph{Proceedings of the IEEE conference on computer vision and
  pattern recognition}, pages 4700--4708, 2017.

\bibitem[Ke et~al.(2021)Ke, Ellsworth, Banerjee, Ng, and
  Rajpurkar]{ke2021chextransfer}
Alexander Ke, William Ellsworth, Oishi Banerjee, Andrew~Y Ng, and Pranav
  Rajpurkar.
\newblock Chextransfer: performance and parameter efficiency of imagenet models
  for chest x-ray interpretation.
\newblock In \emph{Proceedings of the Conference on Health, Inference, and
  Learning}, pages 116--124, 2021.

\bibitem[Hendrycks et~al.(2020)Hendrycks, Basart, Mu, Kadavath, Wang, Dorundo,
  Desai, Zhu, Parajuli, Guo, et~al.]{hendrycks2020many}
Dan Hendrycks, Steven Basart, Norman Mu, Saurav Kadavath, Frank Wang, Evan
  Dorundo, Rahul Desai, Tyler Zhu, Samyak Parajuli, Mike Guo, et~al.
\newblock The many faces of robustness: A critical analysis of
  out-of-distribution generalization.
\newblock \emph{arXiv preprint arXiv:2006.16241}, 2020.

\bibitem[Yin et~al.(2019)Yin, Lopes, Shlens, Cubuk, and Gilmer]{yin2019fourier}
Dong Yin, Raphael~Gontijo Lopes, Jonathon Shlens, Ekin~D Cubuk, and Justin
  Gilmer.
\newblock A fourier perspective on model robustness in computer vision.
\newblock \emph{arXiv preprint arXiv:1906.08988}, 2019.

\bibitem[Lyle et~al.(2020)Lyle, van~der Wilk, Kwiatkowska, Gal, and
  Bloem-Reddy]{lyle2020benefits}
Clare Lyle, Mark van~der Wilk, Marta Kwiatkowska, Yarin Gal, and Benjamin
  Bloem-Reddy.
\newblock On the benefits of invariance in neural networks.
\newblock \emph{arXiv preprint arXiv:2005.00178}, 2020.

\bibitem[Lopes et~al.(2019)Lopes, Yin, Poole, Gilmer, and
  Cubuk]{lopes2019improving}
Raphael~Gontijo Lopes, Dong Yin, Ben Poole, Justin Gilmer, and Ekin~D Cubuk.
\newblock Improving robustness without sacrificing accuracy with patch gaussian
  augmentation.
\newblock \emph{arXiv preprint arXiv:1906.02611}, 2019.

\bibitem[Cubuk et~al.(2018)Cubuk, Zoph, Mane, Vasudevan, and
  Le]{cubuk2018autoaugment}
Ekin~D Cubuk, Barret Zoph, Dandelion Mane, Vijay Vasudevan, and Quoc~V Le.
\newblock Autoaugment: Learning augmentation policies from data.
\newblock \emph{arXiv preprint arXiv:1805.09501}, 2018.

\bibitem[Sagawa et~al.(2020{\natexlab{b}})Sagawa, Raghunathan, Koh, and
  Liang]{sagawa2020investigation}
Shiori Sagawa, Aditi Raghunathan, Pang~Wei Koh, and Percy Liang.
\newblock An investigation of why overparameterization exacerbates spurious
  correlations.
\newblock \emph{arXiv preprint arXiv:2005.04345}, 2020{\natexlab{b}}.

\bibitem[Madras et~al.(2018)Madras, Creager, Pitassi, and
  Zemel]{madras2018learning}
David Madras, Elliot Creager, Toniann Pitassi, and Richard Zemel.
\newblock Learning adversarially fair and transferable representations.
\newblock \emph{arXiv preprint arXiv:1802.06309}, 2018.

\bibitem[Zhang et~al.(2019)Zhang, Yu, Jiao, Xing, El~Ghaoui, and
  Jordan]{zhang2019theoretically}
Hongyang Zhang, Yaodong Yu, Jiantao Jiao, Eric Xing, Laurent El~Ghaoui, and
  Michael Jordan.
\newblock Theoretically principled trade-off between robustness and accuracy.
\newblock In \emph{International Conference on Machine Learning}, pages
  7472--7482. PMLR, 2019.

\bibitem[Calders et~al.(2009)Calders, Kamiran, and
  Pechenizkiy]{calders2009building}
Toon Calders, Faisal Kamiran, and Mykola Pechenizkiy.
\newblock Building classifiers with independency constraints.
\newblock In \emph{2009 IEEE International Conference on Data Mining
  Workshops}, pages 13--18. IEEE, 2009.

\bibitem[Johansson et~al.(2019)Johansson, Sontag, and
  Ranganath]{johansson2019support}
Fredrik~D Johansson, David Sontag, and Rajesh Ranganath.
\newblock Support and invertibility in domain-invariant representations.
\newblock In \emph{The 22nd International Conference on Artificial Intelligence
  and Statistics}, pages 527--536. PMLR, 2019.

\bibitem[Dutta et~al.(2020)Dutta, Wei, Yueksel, Chen, Liu, and
  Varshney]{dutta2020there}
Sanghamitra Dutta, Dennis Wei, Hazar Yueksel, Pin-Yu Chen, Sijia Liu, and Kush
  Varshney.
\newblock Is there a trade-off between fairness and accuracy? a perspective
  using mismatched hypothesis testing.
\newblock In \emph{International Conference on Machine Learning}, pages
  2803--2813. PMLR, 2020.

\bibitem[Zhao and Gordon(2019)]{zhao2019inherent}
Han Zhao and Geoffrey~J Gordon.
\newblock Inherent tradeoffs in learning fair representations.
\newblock \emph{arXiv preprint arXiv:1906.08386}, 2019.

\bibitem[Hardt et~al.(2016)Hardt, Price, and Srebro]{hardt2016equality}
Moritz Hardt, Eric Price, and Nati Srebro.
\newblock Equality of opportunity in supervised learning.
\newblock \emph{Advances in neural information processing systems},
  29:\penalty0 3315--3323, 2016.

\bibitem[Shalit et~al.(2017)Shalit, Johansson, and
  Sontag]{shalit2017estimating}
Uri Shalit, Fredrik~D Johansson, and David Sontag.
\newblock Estimating individual treatment effect: generalization bounds and
  algorithms.
\newblock In \emph{Proceedings of the 34th International Conference on Machine
  Learning-Volume 70}, pages 3076--3085, 2017.

\bibitem[Johansson et~al.(2016)Johansson, Shalit, and
  Sontag]{johansson2016learning}
Fredrik Johansson, Uri Shalit, and David Sontag.
\newblock Learning representations for counterfactual inference.
\newblock In \emph{International conference on machine learning}, pages
  3020--3029. PMLR, 2016.

\bibitem[Tzeng et~al.(2014)Tzeng, Hoffman, Zhang, Saenko, and
  Darrell]{tzeng2014deep}
Eric Tzeng, Judy Hoffman, Ning Zhang, Kate Saenko, and Trevor Darrell.
\newblock Deep domain confusion: Maximizing for domain invariance.
\newblock \emph{arXiv preprint arXiv:1412.3474}, 2014.

\bibitem[Long et~al.(2015)Long, Cao, Wang, and Jordan]{long2015learning}
Mingsheng Long, Yue Cao, Jianmin Wang, and Michael Jordan.
\newblock Learning transferable features with deep adaptation networks.
\newblock In \emph{International conference on machine learning}, pages
  97--105. PMLR, 2015.

\bibitem[Ben-David et~al.(2007)Ben-David, Blitzer, Crammer, Pereira,
  et~al.]{ben2007analysis}
Shai Ben-David, John Blitzer, Koby Crammer, Fernando Pereira, et~al.
\newblock Analysis of representations for domain adaptation.
\newblock \emph{Advances in neural information processing systems},
  19:\penalty0 137, 2007.

\bibitem[Ben-David et~al.(2010)Ben-David, Blitzer, Crammer, Kulesza, Pereira,
  and Vaughan]{ben2010theory}
Shai Ben-David, John Blitzer, Koby Crammer, Alex Kulesza, Fernando Pereira, and
  Jennifer~Wortman Vaughan.
\newblock A theory of learning from different domains.
\newblock \emph{Machine learning}, 79\penalty0 (1):\penalty0 151--175, 2010.

\bibitem[Arjovsky et~al.(2019)Arjovsky, Bottou, Gulrajani, and
  Lopez-Paz]{arjovsky2019invariant}
Martin Arjovsky, L{\'e}on Bottou, Ishaan Gulrajani, and David Lopez-Paz.
\newblock Invariant risk minimization.
\newblock \emph{arXiv preprint arXiv:1907.02893}, 2019.

\bibitem[Guo et~al.(2021)Guo, Zhang, Liu, and Kiciman]{guo2021out}
Ruocheng Guo, Pengchuan Zhang, Hao Liu, and Emre Kiciman.
\newblock Out-of-distribution prediction with invariant risk minimization: The
  limitation and an effective fix.
\newblock \emph{arXiv preprint arXiv:2101.07732}, 2021.

\bibitem[Rosenfeld et~al.(2020)Rosenfeld, Ravikumar, and
  Risteski]{rosenfeld2020risks}
Elan Rosenfeld, Pradeep Ravikumar, and Andrej Risteski.
\newblock The risks of invariant risk minimization.
\newblock \emph{arXiv preprint arXiv:2010.05761}, 2020.

\bibitem[Subbaswamy and Saria(2018)]{subbaswamy2018counterfactual}
Adarsh Subbaswamy and Suchi Saria.
\newblock Counterfactual normalization: Proactively addressing dataset shift
  and improving reliability using causal mechanisms.
\newblock \emph{arXiv preprint arXiv:1808.03253}, 2018.

\bibitem[Simonyan et~al.(2013)Simonyan, Vedaldi, and
  Zisserman]{simonyan2013deep}
Karen Simonyan, Andrea Vedaldi, and Andrew Zisserman.
\newblock Deep inside convolutional networks: Visualising image classification
  models and saliency maps.
\newblock \emph{arXiv preprint arXiv:1312.6034}, 2013.

\bibitem[Lundberg and Lee(2017)]{lundberg2017unified}
Scott Lundberg and Su-In Lee.
\newblock A unified approach to interpreting model predictions.
\newblock \emph{arXiv preprint arXiv:1705.07874}, 2017.

\bibitem[Wang et~al.(2021)Wang, Wiens, and Lundberg]{wang2021shapley}
Jiaxuan Wang, Jenna Wiens, and Scott Lundberg.
\newblock Shapley flow: A graph-based approach to interpreting model
  predictions.
\newblock In \emph{International Conference on Artificial Intelligence and
  Statistics}, pages 721--729. PMLR, 2021.

\bibitem[Gretton et~al.(2007)Gretton, Fukumizu, Teo, Song, Sch{\"o}lkopf,
  Smola, et~al.]{gretton2007kernel}
Arthur Gretton, Kenji Fukumizu, Choon~Hui Teo, Le~Song, Bernhard Sch{\"o}lkopf,
  Alexander~J Smola, et~al.
\newblock A kernel statistical test of independence.
\newblock In \emph{Nips}, volume~20, pages 585--592. Citeseer, 2007.

\bibitem[Ledoux(1996)]{ledoux1996isoperimetry}
Michel Ledoux.
\newblock Isoperimetry and gaussian analysis.
\newblock In \emph{Lectures on probability theory and statistics}, pages
  165--294. Springer, 1996.

\bibitem[Makar et~al.(2020)Makar, Johansson, Guttag, and
  Sontag]{makar2020estimation}
Maggie Makar, Fredrik Johansson, John Guttag, and David Sontag.
\newblock Estimation of bounds on potential outcomes for decision making.
\newblock In \emph{International Conference on Machine Learning}, pages
  6661--6671. PMLR, 2020.

\bibitem[Foster and Syrgkanis(2019)]{foster2019orthogonal}
Dylan~J Foster and Vasilis Syrgkanis.
\newblock Orthogonal statistical learning.
\newblock \emph{arXiv preprint arXiv:1901.09036}, 2019.

\bibitem[Kingma and Ba(2015)]{kingma2015adam}
Diederik~P Kingma and Jimmy Ba.
\newblock Adam: A method for stochastic optimization.
\newblock In \emph{ICLR}, 2015.

\end{thebibliography}
%\newpage
%\input{appendix/appendix}

%%%%%%%%%%%%%%%%%%%%%%%%%%%%%%%%%%%
%%%%%% SUPPLEMENT (OPTIONAL) %%%%%%
%%%%%%%%%%%%%%%%%%%%%%%%%%%%%%%%%%%

\clearpage
\appendix

\thispagestyle{empty}

% For one-column format, uncomment the following:
\onecolumn \makesupplementtitle
% For two-column format, uncomment the following:
%\twocolumn[ \makesupplementtitle ]
% \title{Causally Motivated Shortcut Removal Using Auxiliary Labels: Appendix}

% \maketitle
\section{SOCIETAL IMPACT}
Our approach could be used in certain fairness applications where invariance to a sensitive label is favorable. We caution that even though our approach outperforms the baselines, it is imperfect in that it still exhibits some dependence on the features related to the sensitive (auxiliary) label as shown in figure~\ref{fig:main} (middle). As with most learned models, human audits of the model's output are necessary in such high-stakes applications.

\section{PROOFS FOR SECTION~\ref{sec:prelims}}\label{sec:app_theory}

To prove proposition~\ref{prop:frob_ferm_equal}, we first formally state the assumption of invertability mentioned in section~\ref{sec:prelims}. Specifically, the fact that $\bX^*$ is a function of $\bX$ only implies that $\bX^*$ is invertible, i.e. for all $\bX$, $\bX^*$ can be exactly recovered from $\bX$. We state that formally in the following assumption 
\begin{thmappasmp}\label{asump:invertability}
(Invertability) There exists some function $e$ such that $\bX^* = e(\bX)$ for all $\bX$. 
\end{thmappasmp}

\begin{thmappprop}
[Restated proposition~\ref{prop:frob_ferm_equal}]Under $P^\circ$, the Bayes optimal predictor is (i) only a function of $\bX^*$, and (ii) an optimal risk-invariant predictor $f_{\rinv}$ with respect to $\cP$.
\end{thmappprop}

\begin{proof}
Under $P^\circ$, $\bX^*$ $d$-separates $Y$ from $\bX$, so $\E_{P^\circ}[Y \mid \bX] = \E_{P^\circ}[Y \mid \bX^*]$.
Thus, the population risk minimizer is only a function of $\bX^*$.

By the assumption \ref{asump:invertability}, we have that $\bX^* = e(\bX)$ and hence $\bX^*$ can be perfectly recovered from $\bX$. This means that $\E_{P^\circ}[Y \mid \bX^*]$ can be written as a function of $\bX$, i.e., we can define $g(\bX) = \E_{P^\circ}[Y \mid e(\bX)]$. Thus, the Bayes optimal classifier $f(\bX)$, which is a function of $\E_{P^\circ}[Y \mid \bX] = \E_{P^\circ}[Y \mid e(\bX)]$, can be written (with some abuse of notation) as $f(\bX^*)$ (that is, a function that only varies with the value of $\bX^* = e(\bX)$).

Thus, the risk is also invariant. To see that note the following: 
\begin{align*}
    R^\circ(f) & = \int_{\bX, Y} \ell(f(\bX), Y) P^\circ(\bX \mid \bX^*, V) P^\circ(\bX^* \mid Y) P^\circ(Y) P^\circ(V) dY d\bX\\
    & = \int_{\bX^*, Y} \ell(f(\bX^*), Y)P^\circ(\bX^* \mid Y) P^\circ(Y) dY d\bX^*\\
    & = \int_{\bX^*, Y} \ell(f(\bX^*), Y) P(\bX^* \mid Y) P(Y) dY d\bX^* \\
    & = R_P(f), 
\end{align*}
for any $P \in \cP$.

Because this classifier is optimal under $P^\circ$, no other risk invariant classifier can obtain a lower risk across $\cP$; thus this classifier is an optimal risk invariant classifier.
\end{proof}

\section{PROOFS FOR SECTION~\ref{sec:summary_approach}}

We show that the reweighted risk is an unbiased estimator of the risk under $P^\circ$, i.e., that 
\begin{align*}
    R^\bu_{Ps}(f) = \E_{Ps}\left[\hat{R}^\bu_{Ps}(f)\right] = R_{P^\circ}(f). 
\end{align*}

For any $P_s$, the $\bu$-weighted risk is equal to the risk under the corresponding unconfounded distribution $P^\circ$.
That is, $R_{Ps}^\bu := \E_{Ps}[u(Y,V) \ell(f(\bX, Y))] = R_{P^\circ}$.

To see this, note that the conditional distribution $P_s(\bX \mid Y, V)$ is invariant across the family $\cP$ defined in \eqref{eq:intervention family}.
Thus, the risk conditional on $Y$ and $V$, $R_{Ps\mid y,v} := \E_{Ps}[\ell(f(\bX), Y) \mid Y=y, V=v]$, does not change with $P_s$.
\begin{align*}
R^\bu_{Ps}(f) &:= \E_{Ps}[u(Y,V) \ell(f(\bX, Y))]
= \E_{Ps}[\E_{Ps}[u(Y, V) \ell(f(\bX, Y)) \mid Y=y, V=v]]\\
&= \sum_{y,v} P_s(Y=y,V=v) u(y,v) R_{P \mid y,v} = \sum_{y,v} P_s(Y=y)P_s(V=v) R_{Ps^\circ \mid y,v}\\
&= \E_{P^\circ}[\E_{P^\circ}[R_{P^\circ \mid y,v}]]= R_{P^\circ}.
\end{align*}

\section{PROOFS FOR SECTION~\ref{sec:struc_gap}}

\begin{thmapplem}\label{lem:tau_bounds_components}
For training data $\cD = \{(\bx_i, y_i, v_i)\}_{i=1}^{n}$, $\cD \sim P^\circ$, and a corresponding learned $f = h(\phi(\bx))$ with expected risk $R_{P^\circ}(f)$, suppose that $y$ is $\phi-$representable, i.e., that there exists $g(\phi(\bx)) = y$, and that $g(\phi)\ell(\phi) \in \Omega$. 
Then for all $y$:
\begin{align*}
    P(Y=y)[R^\circ_{0y} - R^\circ_{1y}] \leq \tau, 
\end{align*}
where $R^\circ_{vy}:= \E_{\bx \sim P^\circ}[\ell(f(\bx), y) | V=v, Y=y]$
\end{thmapplem}

\begin{proof}
Without loss of generality, suppose that for $y=1$, 
\begin{align*}
    P(Y=y)[R^\circ_{0y} - R^\circ_{1y}] = \tau_1 > \tau. 
\end{align*}
Then due to the fact that $\mmd \leq \tau$, and by assumption that $\ell \in \Omega$
\begin{align*}
    P(Y=0)[R^\circ_{00} - R^\circ_{10}] &\leq \tau - \tau_1 \\
    P(Y=0)[R^\circ_{00} - R^\circ_{10}] &< 0 \\
    R^\circ_{00} - R^\circ_{10} & < 0. 
\end{align*}
Using the shorthand $R^\circ_{\Delta0} := R_{00} - R_{10}$, the above inequality implies that $-R^\circ_{\Delta0} > 0$. 
\\
\\
\noindent Let $\accentset{\bullet}{\ell}(\phi) = (2g(\phi) -1) \cdot \ell(\phi)$, and $\accentset{\bullet}{R}^\circ:=\E[\accentset{\bullet}{\ell}(\phi)]$. By assumption, we have that $\accentset{\bullet}{\ell}$ is also $\in \Omega$. However, 
\begin{align*}
    \mmd(\accentset{\bullet}{\ell}, P_{\phi_0}, P_{\phi_1}) & = P(Y=0)[\accentset{\bullet}{R}^\circ_{00} - \accentset{\bullet}{R}^\circ_{10}] + P(Y=1)[\accentset{\bullet}{R}^\circ_{01} - \accentset{\bullet}{R}^\circ_{11}] \\
    & = P(Y=0)[- (R^\circ_{00} - R^\circ_{10})] + P(Y=1)[R^\circ_{01} - R^\circ_{11}] \\
    & = P(Y=0) [-R^\circ_{\Delta0}] + \tau_1 > \tau.
\end{align*}
This contradicts the $\mmd$ condition; that for all functions in $\Omega$, $\mmd \leq \tau$

\end{proof}

\begin{thmappprop}
[Restated proposition~\ref{prop:risk_gap}]
Suppose that $f = h(\phi(\bX))$ is a predictor that satisfies $\mmdphi \leq \tau$.
Suppose that $y$ is $\phi-$representable, i.e., that there exists $g(\phi(\bx)) = y$, and that $g(\phi)\ell(\phi) \in \Omega$. 
Then
\begin{align*}
    R_{Pt}(f) < R_{P^\circ}(f) + 2 \tau.
\end{align*}
\end{thmappprop}
\begin{proof}
We will use $P_{v|y}(v) := P(V=v|Y=v)$, $P_y(y) = P(Y=y)$, $P_v(v) = P(V=v)$, and $R^\circ_{vy}$ as defined in lemma~\ref{lem:tau_bounds_components}. 

Note that 
\begin{align*}
    R_{P^\circ} & = \sum_y P_y^\circ(y) \big[ P^\circ_{v|y}(0) R^\circ_{0,y} + P^\circ_{v|y}(1) R^\circ_{1,y} \big]\\
    & = \sum_y P_y(y) \big[ P^\circ_v(0) R^\circ_{0,y} + P^\circ_v(1) R^\circ_{1,y} \big]. 
\end{align*}
And: 
\begin{align*}
    R_P & = \sum_y P_y(y) \big[ P_{v|y}(0) R^\circ_{0,y} + P_{v|y}(1) R^\circ_{1,y} \big].
\end{align*}
Taking the difference between the two:
\begin{align*}
    R_P - R_{P^\circ}
    & = \sum_y P_y(y) \big[ \big(P_{v|y}(0) - P_v^\circ(0)\big) R^\circ_{0,y} + \big(P_{v|y}(1) - P_v^\circ(1)\big) R^\circ_{1,y} \big] \\
    & = \sum_y P_y(y) \big[ \big(P_{v|y}(0) - P_v^\circ(0)\big) R^\circ_{0,y} + \big(\big(1 - P_{v|y}(0)\big) - \big(1 - P_v^\circ(0)\big)\big) R^\circ_{1,y} \big] \\
    & = \sum_y P_y(y) \big[ \big(P_{v|y}(0)- P_v^\circ(0)\big) R^\circ_{0,y} - \big(P_{v|y}(0) - P_v^\circ(0) \big) R^\circ_{1,y} \big] \\
    & = \sum_y P_y(y) \big[ \beta_y R^\circ_{0,y} - \beta_y R^\circ_{1,y} \big] \\
    & = \sum_y \beta_y P_y(y) \big[R^\circ_{0,y} -  R^\circ_{1,y} \big] \\
    & \leq \sum_y \beta_y \tau \\
    & = \beta \cdot \tau
\end{align*}
where in the fourth equality, we use the shorthand $\beta_y := P_{v|y}(0) - P_v^\circ(0)$, and $-1 < \beta_y < 1$. The first inequality follows from lemma~\ref{lem:tau_bounds_components}, and the last equality follows from setting $\beta = \sum_y \beta_y$. Since $\beta \leq 2$, this completes our proof. 
\end{proof}

\section{PROOFS FOR SECTION~\ref{sec:finite_gap}}

\begin{thmappprop}
(Restated proposition \ref{prop:subsets}) For $\cD \sim P^\circ$, and for any for any $\cF_{L_2}$ such that $f_{\text{rinv}} \in \cF_{L_2}$, there exists a $\cF_{\mmd, L_2} \subseteq \cF_{L_2}$ such that $f_{\text{rinv}} \in \cF_{\mmd, L_2}$. And the smallest $\cF_{\mmd, L_2}$ such that $f_\rinv \in \cF_{\mmd, L_2}$ has $\mmd = 0$. 
\end{thmappprop}
\begin{proof}
We prove the existence of a subset,  $\cF_{\mmd, L_2} \subset \cF_{L_2}$ by giving an example of such a subset. Consider 
\begin{align*}
   \cF_{L_2, \mmd} = \{f: \bx \mapsto \sigma(\bw^\top \bx), &\|\bw\|_2 \leq A,
\ \mmd(P^\circ_{\boldsymbol{\phi}_0}, P^\circ_{\boldsymbol{\phi}_1}) = 0 \}, 
\end{align*}
Clearly, $\cF_{\mmd, L_2} \subset \cF_{L_2}$. We will now show that any $f_\rinv \in \cF_{L_2}$ is also $\in \cF_{\mmd, L_2}$. 
\\
\\
\noindent By the definition of $f_\rinv$, any $f_\rinv \in \cF_{L_2}$ must satisfy $f_\rinv(\bx) \indep v$.
Then $T_1(f_\rinv(\bx)) \indep T_2(v)$ for any transformations $T_1, T_2$. Taking $T_1$ to be the inverse of the sigmoid function, $\sigma^{-1}$, and $T_2$ to be the identity transformation, we get that $\sigma^{-1}(f_\rinv(\bx)) = \sigma^{-1}(\sigma(\bw^\top \bx)) = \bw^\top \bx \indep v$. This implies that $p(\bw^\top \bx | v = 0) = p(\bw^\top \bx | v = 1)$, which in turn implies that $\mmd(P^\circ_{\boldsymbol{\phi}_0}, P^\circ_{\boldsymbol{\phi}_1}) = 0$, where $\phi(\bx) = \bw^\top \bx $. 
\end{proof}

\begin{thmappprop} 
(Restated proposition~\ref{prop:wdelta_bound}). Let $f(\bx) = \sigma(\phi(\bx)) = \sigma(\bw^\top \bx)$ be a function contained in $\cF_{L_2, \mmd}$. Then,
\begin{align}
\|\bw_\perp\| \leq \frac{\tau}{\|\Delta\|}.
\end{align}
\end{thmappprop}
\begin{proof}
Note that $\tau$ must be greater than 0. If not, let $\omega'$ be the function that achieves the max difference $\tau' <0$. Then we can define $\omega'' = -\omega'$, which achieves $\tau'' = - \tau' >0$, which is a contradiction. This means that for all $\omega \in \Omega$,
\begin{align*}
    \tau \geq \left|\E[\omega(\bx_i) \mid v_i = 0 ] - \E[\omega(\bx_i) \mid v_i = 1] \right|
\end{align*}
Taking $\omega(\bx) = \bw^\top \bx$, 
\begin{align*}
    \tau \geq \left|\E[\bw^\top \bx_i \mid v_i = 0 ] - \E[\bw^\top \bx_i \mid v_i = 1]\right| = \left| \bw^\top\Delta\right|.
\end{align*}
Note that $\|\bw_\perp\| = \frac{|\bw^\top\Delta|}{\|\Delta\|}$, which completes our proof. 
\end{proof}

\begin{thmappprop}
(Restated Proposition~\ref{prop:rad_fc_fm}).
Let $\bx_\perp := \Pi \bx$, $\bx_\parallel := (I-\Pi)\bx$.
For training data $\cD = \{(\bx_i, y_i, v_i)\}_{i=1}^{n}$, $\cD \sim P^\circ$,  
$\sup_{\bx_\perp} \|\bx_\perp\|_2 \leq B_\perp$,
$\sup_{\bx_\parallel} \|\bx_\parallel \|_2 \leq B_\parallel$,
\begin{align*}
    \mathfrak{R}(\cF_{L_2}) \leq \frac{A \sqrt{B_\parallel^2 + B_\perp^2}}{\sqrt{n}},
\end{align*}
and 
\begin{align*}
    \mathfrak{R}(\cF_{\mmd,L_2}) \leq \frac{A \cdot B_\parallel + \tau \frac{B_\perp}{\|\Delta\|}}{\sqrt{n}}.
\end{align*}
\end{thmappprop}

\begin{proof}
First, we derive the bound on $\mathfrak{R}(\cF_{L_2})$
\begin{align*}
& \mathfrak{R}(\cF_{L_2}) = \E_\cD \E_\epsilon \bigg[ \underset{\bw:\|\bw\|_2 \leq A}{\sup} \  
    \frac{1}{n}\sum_i \epsilon_i \bw^\top \bx_i\bigg]\\
& = \E_\cD \E_\epsilon \bigg[ \underset{\bw:\|\bw\|_2 \leq A}{\sup} \  
    \frac{1}{n}\sum_i \epsilon_i \bw^\top (\bx^{}_{\perp i} + \bx^{}_{\parallel i}) \bigg]
\end{align*}
Following the usual derivations (e.g., see \cite{mohri2018foundations}), we get the desired result for $\mathfrak{R}(\cF_{L_2})$. Next, we derive the bound on $\mathfrak{R}(\cF_{\mmd,L_2})$. 

\begin{align*}
& \mathfrak{R}(\cF_{\mmd,L_2}) = \E_\cD \E_\epsilon \bigg[ \underset{\bw:\|\bw\|_2 \leq A}{\sup} \  
    \frac{1}{n}\sum_i \epsilon_i \bw^\top \bx_i\bigg]\\
& = \E_\cD \E_\epsilon \bigg[\underset{\bw:\|\bw\|_2 \leq A}{\sup}  \ 
    \frac{1}{n}\sum_i \epsilon_i \big(\Pi \bw^\top \bx_i + (1 - \Pi) \bw^\top \bx_i \big) \bigg]\\
& \leq \E_\cD \E_\epsilon 
    \bigg[\sup_{\substack{\bw_\parallel:\|\bw_\parallel\|_2 \leq A \\ \bw_\perp:\|\bw_\perp\|_2 \leq A}} \  
    \frac{1}{n}\sum_i \epsilon_i \bw_\perp^\top \bx^{}_{\perp i} + \epsilon_i \bw_\parallel^\top \bx^{}_{\parallel i} \bigg]\\
& \leq \E_\cD \E_\epsilon \bigg[ \underset{\bw_\perp:\|\bw_\perp\|_2 \leq A}{\sup} \
    \frac{1}{n}\sum_i \epsilon_i \bw_\perp^\top \bx^{}_{\perp i} \bigg] 
 + \E_\cD \E_\epsilon \bigg[ \underset{\bw_\parallel:\|\bw_\parallel\|_2 \leq A}{\sup} \ 
    \frac{1}{n}\sum_i  \epsilon_i \bw_\parallel^\top \bx^{}_{\parallel i} \bigg], 
\end{align*}
where the last inequality follows by the subadditivity of the supremum. Again, following the usual derivations (e.g., see \cite{mohri2018foundations}), we get the required result for $\mathfrak{R}(\cF_{\mmd,L_2})$
\end{proof}

In the following proposition we show that when sampling from a biased distribution, the smallest $\tau'$ that does not introduce bias is greater than 0. 

\begin{thmappprop}\label{prop:looser_mmd}
Let $\cF'_{L_2, \mmd}:=  \{f: \bx \mapsto \sigma(\bw^\top \bx), \|\bw\|_2 \leq A,
\ \mmd(P_{\boldsymbol{\phi}_0}, P_{\boldsymbol{\phi}_1}) \leq \tau'\}$ be the smallest function class that contains $f_\rinv$. Then $\tau' = c \cdot A$ for some $c > 0$, and the corresponding upper bound on the Rademacher complexity of $\cF'_{\mmd,L_2}$ is 
\begin{align*}
    \mathfrak{R}(\cF'_{\mmd,L_2}) \leq \frac{A \cdot B_\parallel + c \cdot A \frac{B_\perp}{\|\Delta\|}}{\sqrt{n}}.
\end{align*}
\end{thmappprop}
\begin{proof}
By proposition~\ref{prop:subsets}, we have that the smallest $\mmd$ regularized function class that contains $f_\rinv$ when $\cD \sim P^\circ$ has $\mmd =0$. And by proposition~\ref{prop:wdelta_bound} we have in that function class $\|\bw_\perp \| =0$, i.e., $\bw_\perp$ is the 0 vector.

\begin{align*}
     \tau' & \geq \left\|\E[\bw^\top \bx_i \mid v_i = 0 ] - \E[\bw^\top \bx_i \mid v_i = 1]\right\| \\
    & = \left\|\E[\bw_\perp^\top \bx_{\perp i} +  \bw_\parallel^\top \bx_{\parallel i} \mid v_i = 0 ] - \E[\bw_\perp^\top \bx_{\perp i} +  \bw_\parallel^\top \bx_{\parallel i} \mid v_i = 1] \right\| \\
    & = \left\|\E[\bw_\parallel^\top \bx_{\parallel i} \mid v_i = 0 ] - \E[\bw_\parallel^\top \bx_{\parallel, i} \mid v_i = 1] \right\| \\
    & = \Big\|\bw_\parallel \big(\E[\bx_{\parallel i} \mid v_i = 0 ] - \E[\bx_{\parallel, i} \mid v_i = 1]\big) \Big\|\\
    & = \Big\|\bw_\parallel (1- \Pi) \big(\E[\bx_i \mid v_i = 0 ] - \E[\bx_i \mid v_i = 1]\big) \Big\|\\
    & = \|\bw_\parallel\| \left\|(1- \Pi) \big(\E[\bx_i \mid v_i = 0 ] - \E[\bx_i \mid v_i = 1]\big) \right\|\\
    & = A \left\|(1- \Pi) \Delta_P\right\|, 
\end{align*}
where the fifth equality holds because the two vectors are scalar multiples of the same vector (they are both projections onto the vector orthogonal to $\Delta$) so Cauchy-Schwartz holds with equality. Also note that $\left\|(1- \Pi) \Delta_P\right\| = 0$ if and only if $\Delta_P = \Delta$, i.e., $P = P^\circ$. So $\left\|(1- \Pi) \Delta_P\right\| > 0$. 

The upper bound on the Rademacher complexity follows immediately by plugging the upper bound on $\tau'$ into $\mathfrak{R}(\cF_{\mmd, L_2})$ obtained in proposition~\ref{prop:rad_fc_fm}.
\end{proof}

\paragraph{Notes on proposition~\ref{prop:looser_mmd}}
\begin{enumerate}
    \item One important implication of proposition~\ref{prop:looser_mmd} is that naively implementing the $\mmd$ penalty when the data are sampled from some $P_s \not= P^\circ$ will lead to a bias-invariance tradeoff. For any $\tau' < c \cdot A$, invariance comes at the cost of bias. To understand why, note that when $\tau' < c \cdot A$, the optimal function does not exist in the candidate function class. Recall that when sampling from $P^\circ$, we could allow $\tau$ to go to zero without introducing bias. 
    \item Comparing the upper bound on $\mathfrak{R}(\cF_{\mmd, L_2})$ in proposition~\ref{prop:rad_fc_fm} to $\mathfrak{R}(\cF'_{\mmd, L_2})$ in proposition~\ref{prop:looser_mmd} shows that the statistical complexity increases when sampling from $P_s \not= P^\circ$. Consequently, the generalization error is higher (i.e., less favorable) when sampling from a non-ideal, correlated distribution. 
\end{enumerate}

\subsection{Full generalization error statements}\label{sec:full_ge_statements}

We start with the simpler case, where $\cD \sim P^\circ$. 

\begin{thmappprop}
For a dataset $\cD \sim P^\circ$, and loss bounded above by $M >0$, the finite-sample gap 
\begin{align*}
    R_{P^\circ}(f) - \hat{R}_{P^\circ}(f) \leq   \frac{A \cdot B_\parallel + \tau \frac{B_\perp}{\|\Delta\|}}{\sqrt{n}} + M \sqrt{\frac{\log \frac{1}{\delta}}{2n}}, 
\end{align*}
\begin{proof}
    The proof is a straightforward application of theorem 11.3 in \citet{mohri2018foundations}. The result can be immediately acquired by substituting $\mu$ in theorem 11.3 with 1 (since the logistic loss is 1-Lipschitz), and $\mathfrak{R}_m(\cH)$ with $\mathfrak{R}(\cF_{\mmd,L_2})$ from proposition~\ref{prop:rad_fc_fm}. 
\end{proof}

\end{thmappprop}

To get a similar statement for the case where $\cD \sim P_s \not=P^\circ$, and reweighting is needed, we need to apply the techniques for estimating the generalization error of reweighted estimators presented in~\citet{cortes2010learning}. To apply the Cortes results, we need to construct a discretization or a covering of $\cF_{L_2, \mmd}$, defined next. 
\begin{thmappdef}\label{def:discrete_f}
    Given any function class $\cF$, a metric $D$ on the elements of $\cF$, and $\varepsilon >0$, we define a covering number $\cN(\cF, D, \varepsilon)$ as the minimal number $m$ of functions $f_1, f_2, \dots, f_m \in \cF$, such that for all $f \in \cF$, $\min_{i=1,\dots,m} D(f_i, f) \leq \varepsilon$, with
    \begin{align*}
        D(f, f') = \sqrt{\frac{1}{n} \sum_i (f(\bx_i) - f'(\bx_i))^2}.
    \end{align*}
\end{thmappdef}

% --- gaussian complexity 
Our statement also makes use of Gaussian complexities, defined next. 
\begin{thmappdef}
    For a function family $\cF$, the empirical Gaussian complexity is defined as: 
    \begin{align*}
        \mathfrak{G}(\cF) = \E_\cD\E_\eta \bigg[ \underset{f \in \cF}{\sup} \eta_i f(\bx_i) \bigg]
    \end{align*}
\end{thmappdef}

% ----metric entropy of the discretized function space
We are now ready to present the metric entropy of the discretized hypothesis space in this next lemma. 
\begin{thmapplem}\label{LEM:METRIC_ENTROPY_FC}
Let $\bx_\perp := \Pi \bx$, $\bx_\parallel := (I-\Pi)\bx$, 
$\sup_{\bx_\perp} \|\bx_\perp\|_2 \leq B_\perp$,
$\sup_{\bx_\parallel} \|\bx_\parallel \|_2 \leq B_\parallel$, $D, \varepsilon$ as is defined in~\ref{def:discrete_f}. For $\varepsilon, c', c'' >0$:
    \begin{align*}
        & \log(\cN(\cF_{L_2, \mmd}, D, \varepsilon)) \\
        & \leq c'' \left(\frac{c' \sqrt{\log(n)} \cdot \left(A \cdot B_\parallel + \tau \frac{B_\perp}{\|\Delta\|}\right)}{\varepsilon} \right)^2 \\
    \end{align*}
\end{thmapplem}
\begin{proof}
    We construct our argument relying on Sudakov's minoration, and the bound between Gaussian and Rademacher complexities. Specifically, by \cite{ledoux1996isoperimetry}, for some $c' >0$: 
    \begin{align*}
        \mathfrak{G}_m(\cF_{L_2, \mmd}) & \leq c' \sqrt{\log(n)} \cdot \mathfrak{R}(\cF_{L_2, \mmd}) \\
        & \leq c' \sqrt{\log(n)} \cdot \frac{A \cdot B_\parallel + \tau \frac{B_\perp}{\|\Delta\|}}{\sqrt{n}}, 
    \end{align*}
    where the last inequality follows from plugging in the results from proposition~\ref{prop:rad_fc_fm}. By Sudakov's minoration (see \cite{ledoux1996isoperimetry} theorem 3.18), for some universal constant $c'' > 0$, 
    \begin{align*}
        & \log(\cN(\cF_{L_2, \mmd}, D, \varepsilon))  \leq c'' \bigg(\frac{\sqrt{n} \cdot \mathfrak{G}_m(\cF_{L_2, \mmd})}{\varepsilon} \bigg)^2 \\
        & \leq c'' \left(\frac{c' \sqrt{\log(n)} \cdot \left(A \cdot B_\parallel + \tau \frac{B_\perp}{\|\Delta\|}\right)}{\varepsilon} \right)^2 \\
    \end{align*}
\end{proof}

Finally, to use the~\citet{cortes2010learning} results, we need bounded divergence between the source and the target distribution. 
Recall that we are trying to bound the finite-sample gap, so the source distribution is $P_s$, and the target distribution is $P^\circ$. We can bound this divergence because of the overlap assumption stated in section~\ref{sec:prelims}. As a consequence of the overlap assumption, we have that:
\begin{align}\label{eqn:cortes_mu}
    \sup u(y,v) = \sup\frac{P^\circ(Y \mid V)}{P_s(Y \mid V)} = 2^{\Xi_\infty(P^\circ|| P_s)} = C_{Ps}, 
\end{align}
where $\Xi_k(p || q)$ is the k$^{\text{th}}$-order R\'enyi divergence, and the second equality follows by applying the Bayes rule, and the definition of the R\'enyi divergence. It will be convenient to denote $2^{\Xi_k(p||q)}$ by $\Lambda_k(p||q)$. 
%For $k=2$, the R\'enyi divergence coincides with the relative entropy. 
Since $2^{\Xi{k-1}(P^\circ|| P_s)} < 2^{\Xi_k(P^\circ|| P_s)}$, we have $\Lambda_2(P^\circ|| P_s) < C_{Ps}$.
Following similar work (e.g., \cite{makar2020estimation}), we will assume that the weights $\bu$ are known, or can be perfectly estimated from the data. In other words, we do not consider estimation error that might arise because of poor estimation of $\bu$. Work by \cite{foster2019orthogonal} has shown that under mild assumptions, the error due to estimation of $\bu$ from finite samples only results in a fourth order dependence in the final classifier, and hence does not greatly affect our derived generalization bounds.

With that, we are ready to state the finite-gap bound when $P_s \not= P^\circ$, and reweighting is used. 
\begin{thmappprop}\label{prop:reweighted_gen}
For $\cD \sim P$, with $P \in \cP$, and $\mathbf{u}$ as defined in equation \ref{eq:weights}, $C_P$ as defined in equation \ref{eqn:cortes_mu},  
\begin{align*}
    R_{P^\circ}(f) -
    \hat{R}^\bu_{P}(f) \leq 
    \frac{2C_{P}(\kappa(\cF_{\mmd, L_2}) + \log \frac{1}{\delta})}{2n} + \sqrt{\frac{\Lambda(P^\circ || P)\cdot (\kappa(\cF) + \log \frac{1}{\delta})}{n}}, 
\end{align*}
where
\begin{align*}
    \kappa(\cF_{\mmd, L_2}) = c'' \left(\frac{c' \sqrt{\log(n)} \cdot \left(A \cdot B_\parallel + \tau \frac{B_\perp}{\|\Delta\|}\right)}{\varepsilon} \right)^2
\end{align*}
\end{thmappprop}
\begin{proof}
    Using the bound on the metric entropy derived in lemma~\ref{LEM:METRIC_ENTROPY_FC}, the proof becomes a direct application of Theorem 2 in \cite{cortes2010learning}.
\end{proof}

% \section{Experiment details}
% We pick 300 water images that are predominantly water. For example, we remove water images that have large lakes reflecting the surrounding trees since that can be indistinguishable from land images. We also remove images that have birds, since these might conflict with the bird image that we superimpose in our data generation process. We do the same with land images. 

% We augment the background images by applying combinations of the following transformations: 

\section{EXPERIMENTS}\label{sec:experiments_extra}

\subsection{Waterbirds: example images}
Figure~\ref{fig:bird_examples} shows examples of the generated images. 

\newcommand{\subf}[2]{%
  {\small\begin{tabular}[t]{@{}c@{}}
  #1\\#2
  \end{tabular}}%
}
	
\begin{figure}[H]
\centering
\begin{tabular}{|c|c|}
\hline
\subf{\includegraphics[width=60mm]{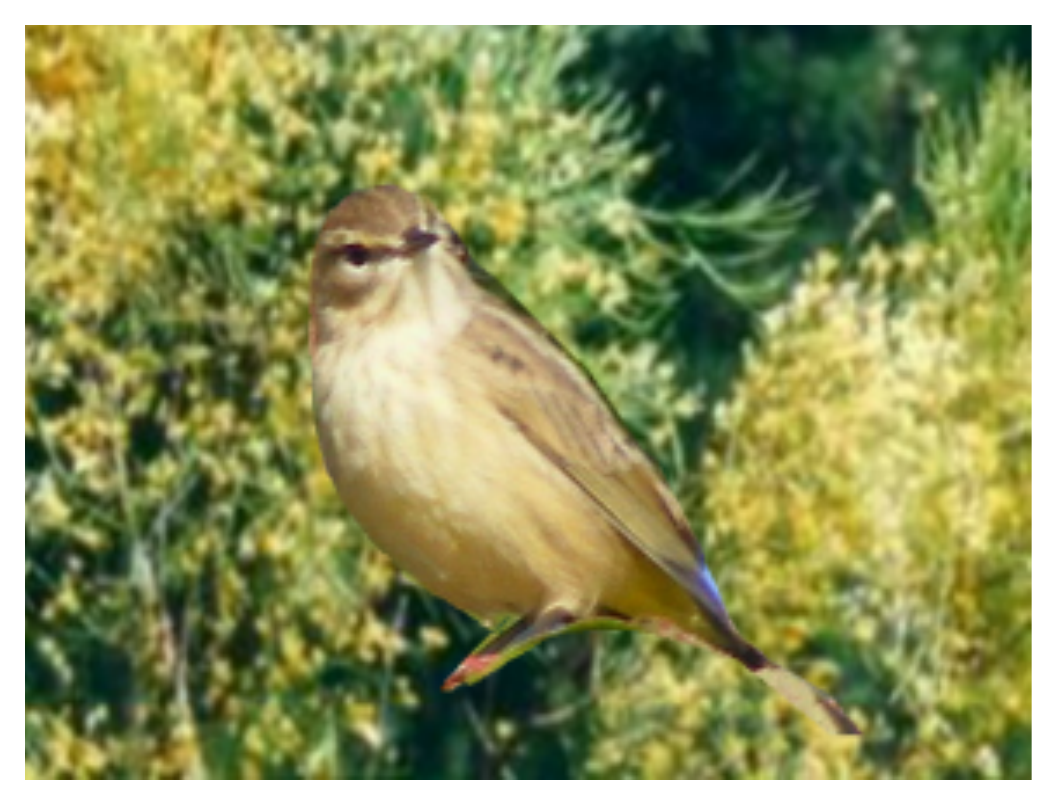}}
     {Land bird on land background \\ $y=0, v=0$}
&
\subf{\includegraphics[width=60mm]{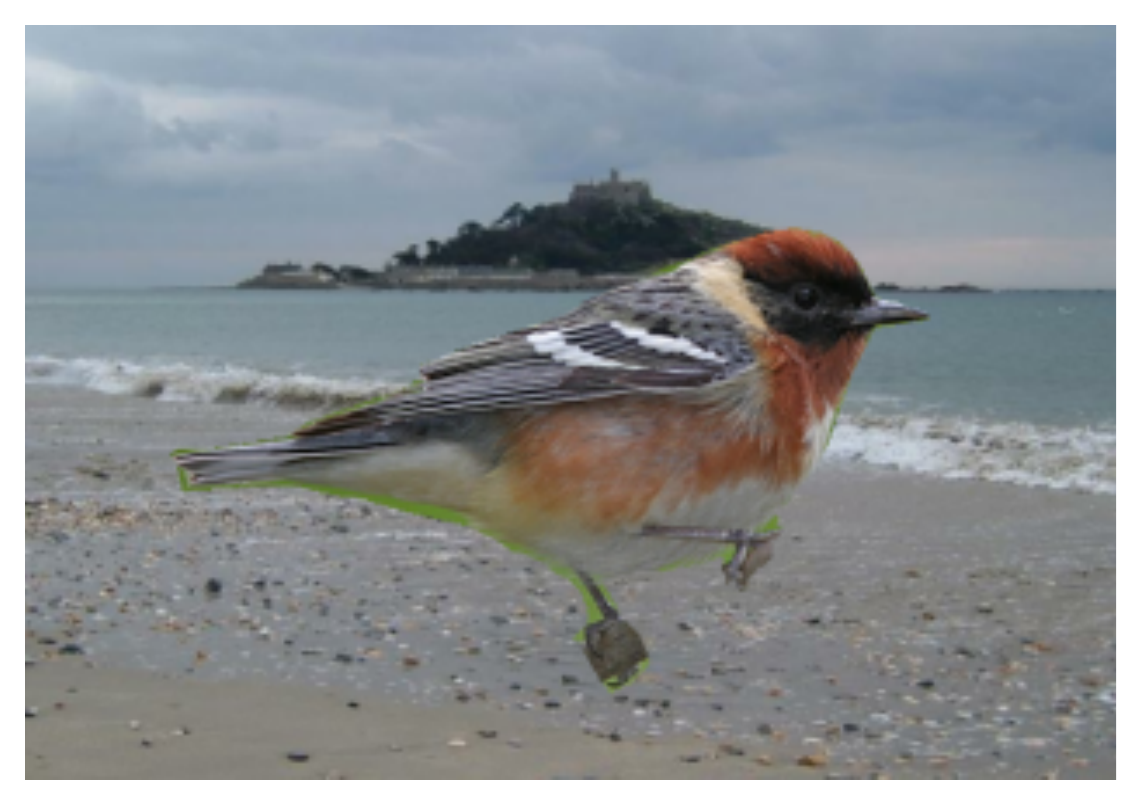}}
     {Land bird on water background \\ $y=0, v=1$}
\\
\hline
\subf{\includegraphics[width=60mm]{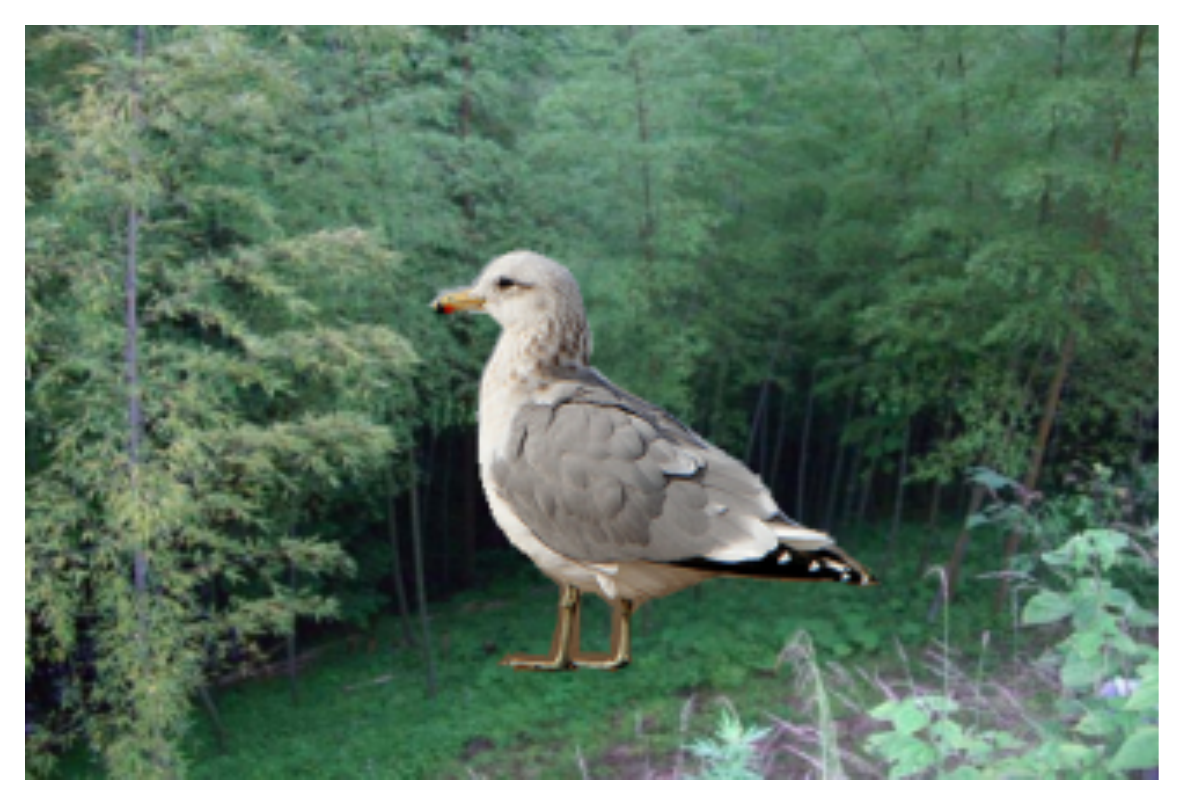}}
     {Water bird on land background \\ $y=1, v=0$}
&
\subf{\includegraphics[width=60mm]{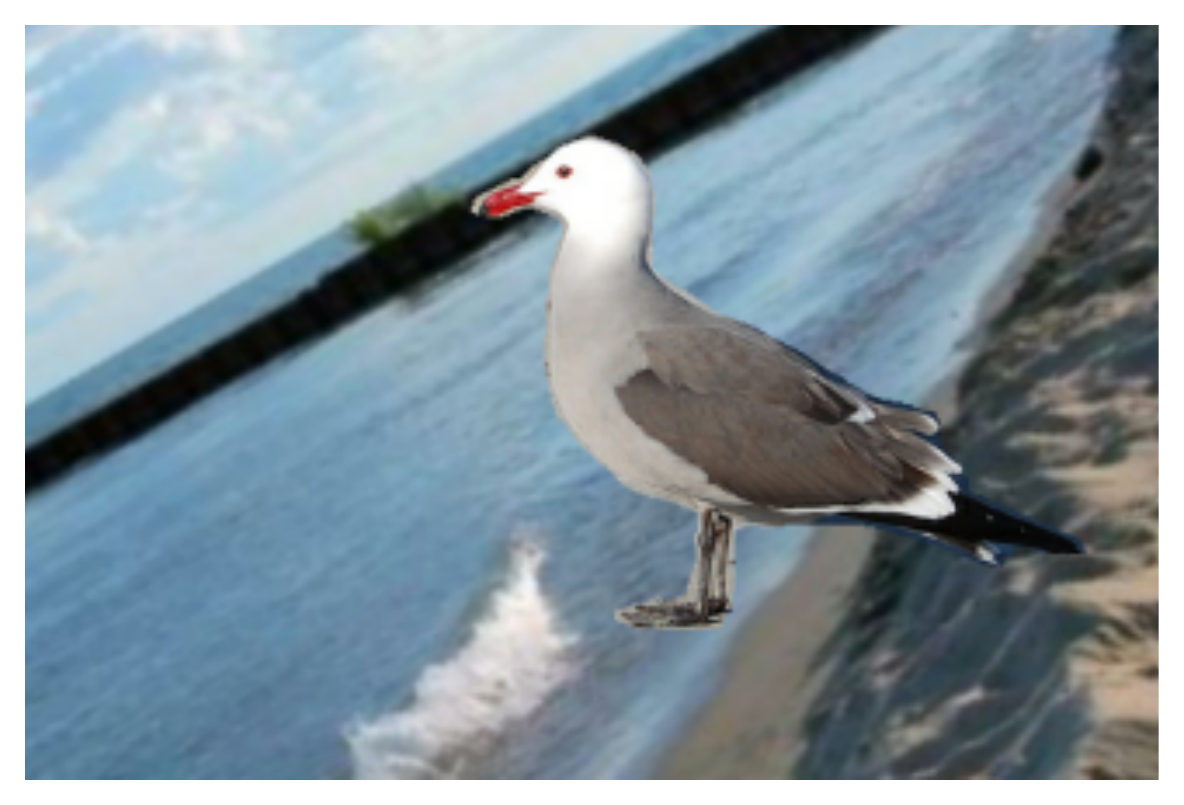}}
     {Water bird on water background \\ $y=1, v=1$}
\\
\hline
\end{tabular}
\caption{Examples of the generated images of water, and land birds on water, and land backgrounds \label{fig:bird_examples}}
\end{figure}

\subsection{Hyperparameter setting}

We train all models using Adam \cite{kingma2015adam}, with learning rate $=0.001$, $\beta_1=0.9$, $\beta_2=0.999$. 

We train all models for 200 epochs in the waterbirds setting, and for 10 epochs in the CheXpert setting. 
For $L_2$ regularized models, we cross-validate the $L_2$ penalty parameter from the following values $[0.0$ (no regularization)$, 0.001, 0.0001]$, which is similar to values typically used for this setting \citep{sagawa2019distributionally,he2016deep}. For the $\mmd$ regularized models, we pick the optimal $\mmd$ penalty parameter, $\alpha$ from $[1e^3, 1e^5, 1e^7]$. We also pick the best RBF kernel bandwidth from $[1e1, 1e2, 1e3]$. 

For our two-step cross-validation, we determine that a model has an $\mmd$ that is statistically insignificantly different from 0 by using a standard T-test. Specifically, we compare the mean $\mmd$ across 5 folds of validation data to 0. If the P-value of that T-test is greater than 0.05, we say that the model has an $\mmd$ that is statistically insignificantly different from 0. 

\subsection{Waterbirds: Noisy background}
We found that the original background images frequently contain landscapes that are difficult to distinguish (e.g., water backgrounds with very small water bodies that mostly reflect the surrounding trees). To address that, we pick 200 and 300 ``clean'' images for each of the land and water backgrounds respectively. Using those clean images, we generate 10,000 land backgrounds, and 9,000 water backgrounds by applying random transformations (rotation, zoom, darkening/brightening) to the selected images. Figure~\ref{fig:background_examples} shows examples of the excluded backgrounds. To protect the privacy of the individuals in the pictures, their faces have been blurred.

\begin{figure}[H]
\centering
\begin{tabular}{|c|c|}
\hline
\subf{\includegraphics[width=60mm]{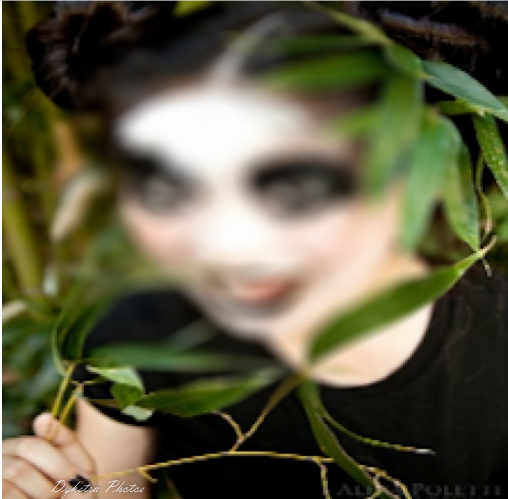}}
     {Excluded land background}
&
\subf{\includegraphics[width=60mm]{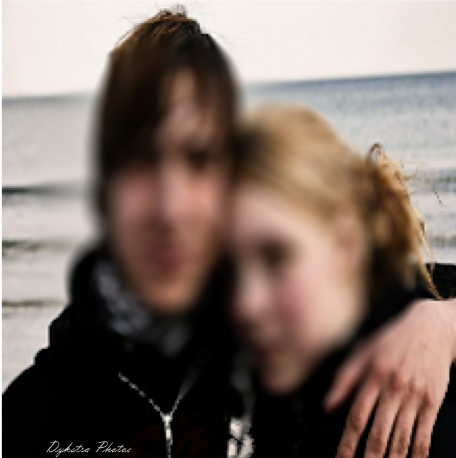}}
     {Excluded water background}
\\
\hline
\end{tabular}
\caption{Examples of excluded noisy backgrounds. Pictures extracted from the Places dataset \cite{zhou2017places}. Blurring is added to preserve privacy. \label{fig:background_examples}}
\end{figure}

Here we present the results using the full (noisy) background images. Results from the main analysis largely hold, with two distinctions. First, in the ideal distribution, because the backgrounds are noisy, we see an overall higher variance in performance, so the models perform equally well with no clear ``winner''. Second, we note that the cMMD-T is particularly sensitive to the level of noise in the backgrounds, while other models are overall more robust.

\begin{figure}[H]
\centering
  \includegraphics[width=0.8\textwidth]{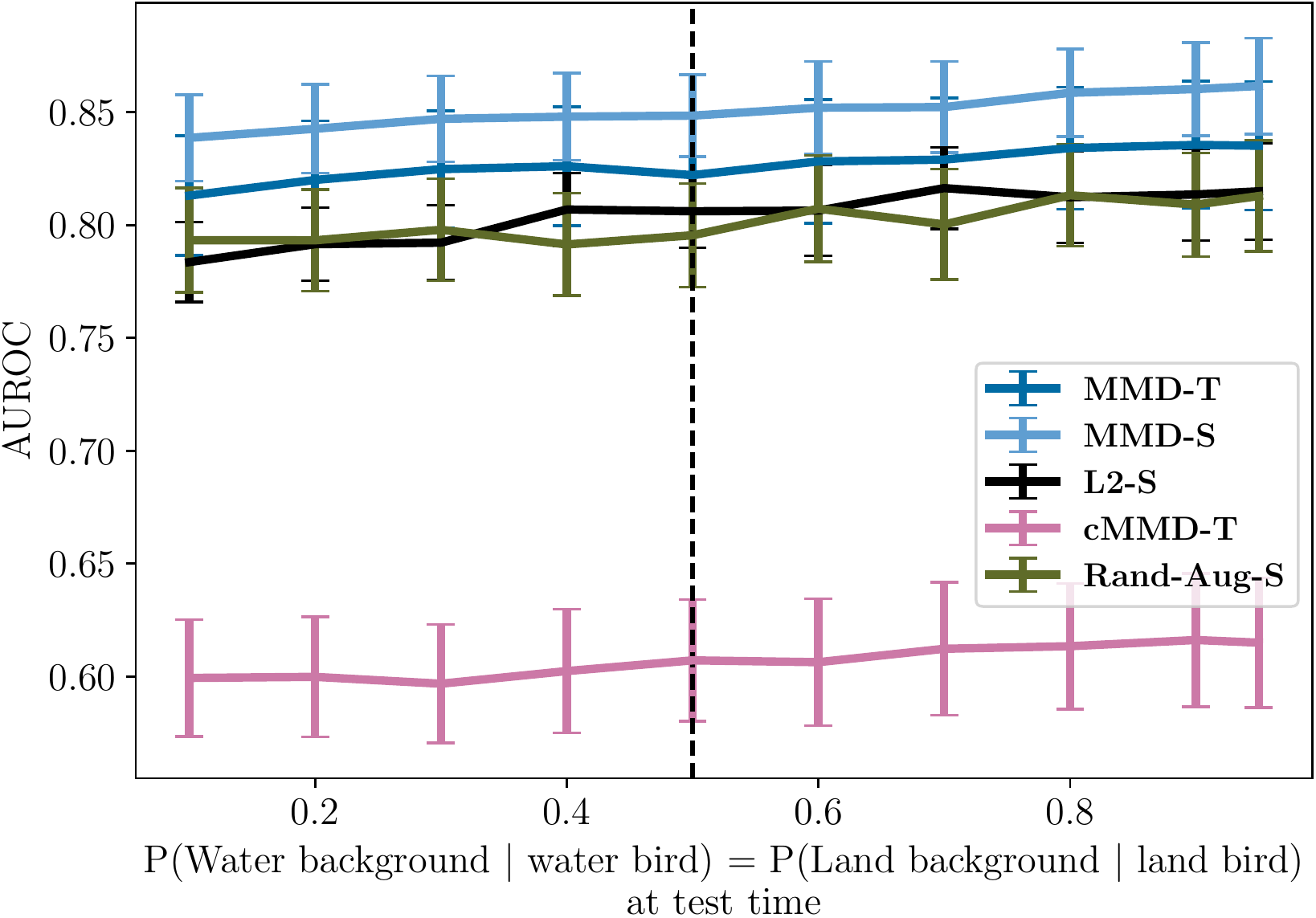}
  \caption{Training data sampled from $P^\circ$, with $P^\circ(Y|V=1) = P^\circ(Y|V=0) = 0.5$ and backgrounds are sampled from a noisy set of images. $x$-axis shows $P(Y|V)$ at test time under different shifted distributions. $y$-axis shows AUROC on test data. Vertical dashed line shows training data. MMD-regularized models outperform baselines within, and outside the training distribution.\label{fig:main_a_app}}
\end{figure}

\begin{figure}[H]
\centering
\includegraphics[width=0.8\textwidth]{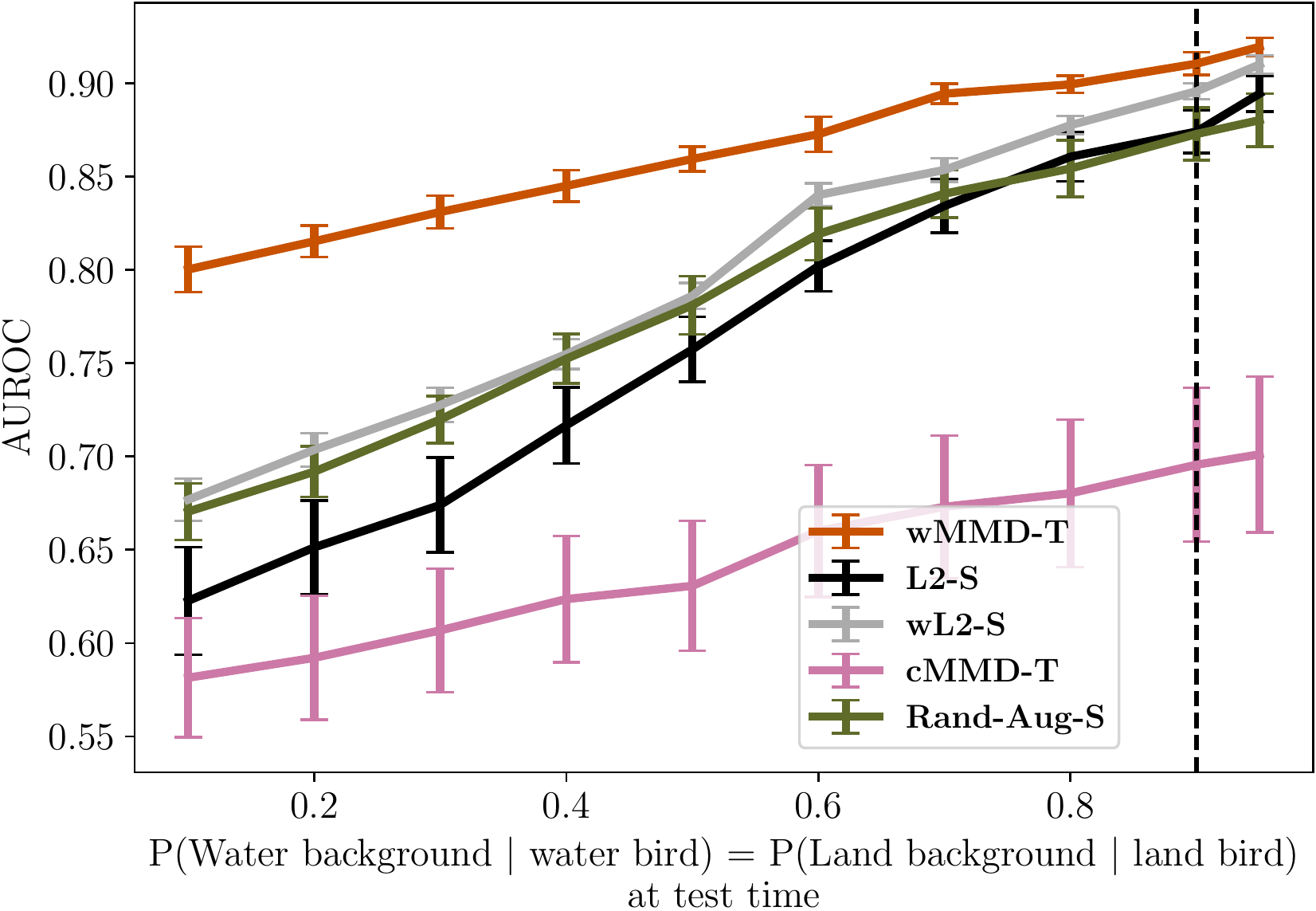}
\caption{Training data sampled from $P$, with $P(Y=1|V=1) = P^\circ(Y=0|V=0) = 0.9$, and backgrounds are sampled from a noisy set of images. Vertical dashed line shows training data. $x$, $y$ axes similar to figure~\ref{fig:main_a_app}. MMD-regularized models outperform baselines showing better robustness against distribution shifts at test time.\label{fig:main_b_app}}
\end{figure}

\begin{figure}[H]
\centering
\includegraphics[width=0.8\textwidth]{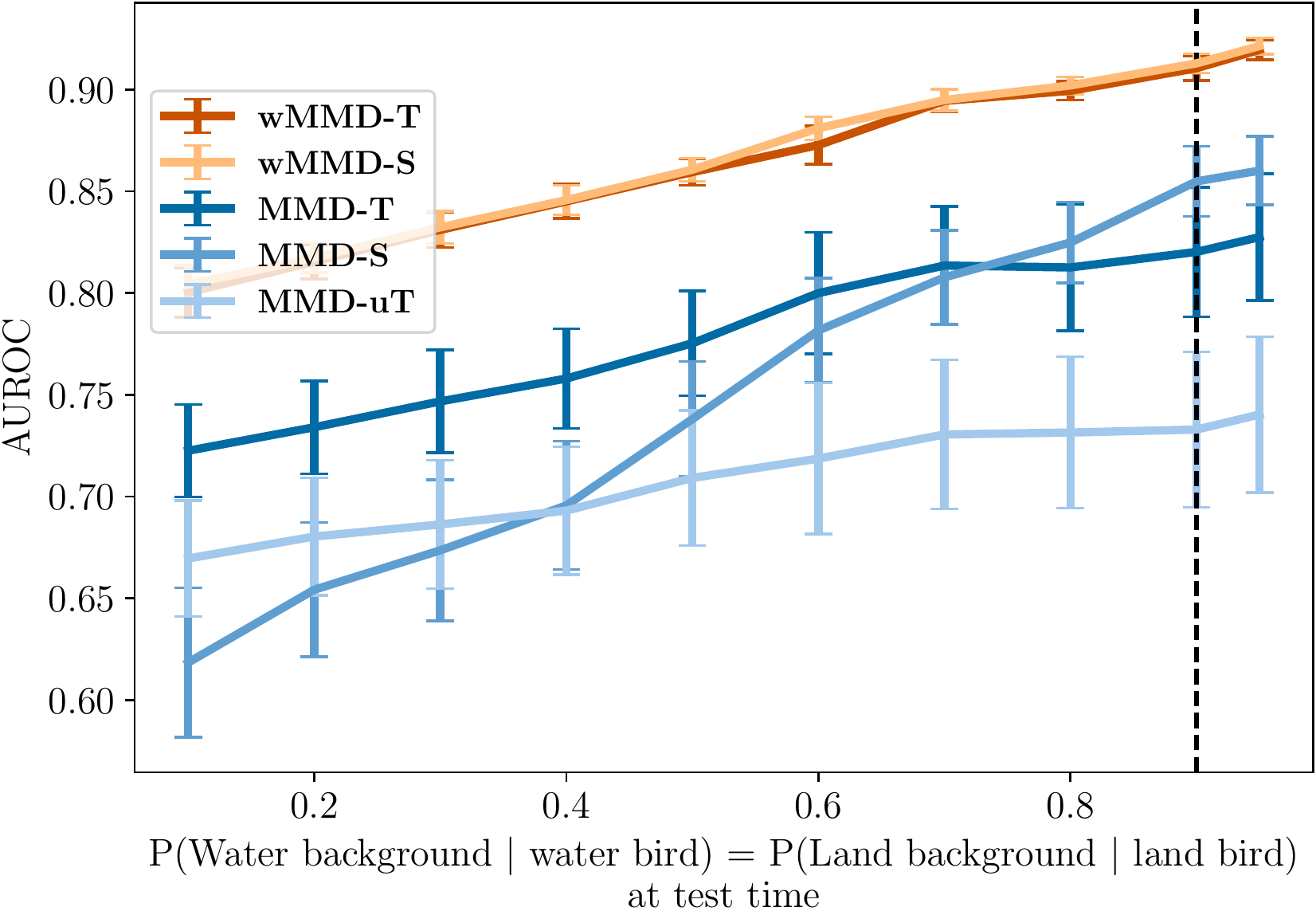}
\caption{Training data sampled from $P$, with $P(Y=1|V=1) = P^\circ(Y=0|V=0) = 0.9$. $x$, and backgrounds are sampled from a noisy set of images. $y$ axes similar to fig~\ref{fig:main_a_app}.
An ablation study to show how different components of our suggested approach (wMMD-reg-T) contribute to improved performance. \label{fig:ablation_app}}
\end{figure}

% \begin{figure}[H]
% \centering
% \includegraphics[width=0.8\textwidth]{plots/oracle_8090_noisyback.pdf}
% \caption{
% Similar setting to figure~\ref{fig:ablation_app}. We compare our main approach (wMMD-reg-T) to unrealistic baselines that have oracle knowledge of the the optimal image transformations at training time. Our approach performs comparably to an augmentation method that has 50\% of its training images sampled from $P^\circ$, and has better performance than an oracle model that samples only 10\% of its training data from $P^\circ$.\label{fig:oracle_app}}
% \end{figure}

\subsection{Waterbirds: Other performance metrics}
We re-examine the results from figure~\ref{fig:main} (middle) considering performance criteria other than the AUROC. Specifically, in figure~\ref{fig:logloss_appendix}, we consider the logistic loss on the $y$-axis whereas in figure~\ref{fig:brierscore_appendix} we consider the Brier score on the $y$-axis. In both cases, lower is better. The results from both plots conform with our findings from figure~\ref{fig:main} (middle). 

\begin{figure}[H]
\centering
\includegraphics[width=\textwidth]{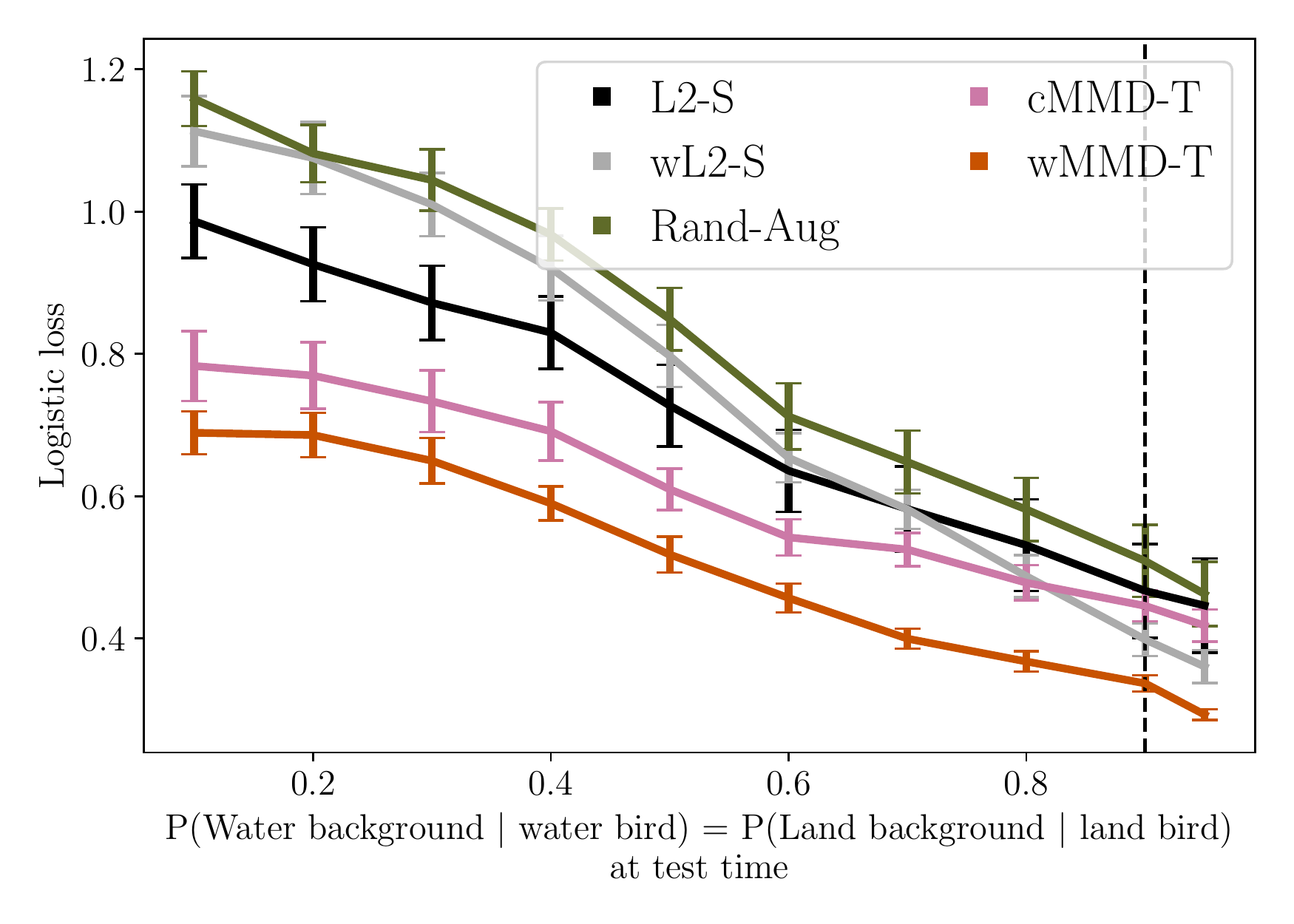}
\caption{
$x$-axis shows $P(Y|V)$ at test time under different shifted distributions, $y$-axis shows logistic loss on test data, and vertical dashed line shows $P(Y|V)$ at training time.
\label{fig:logloss_appendix}}
\end{figure}

\begin{figure}[H]
\centering
\includegraphics[width=\textwidth]{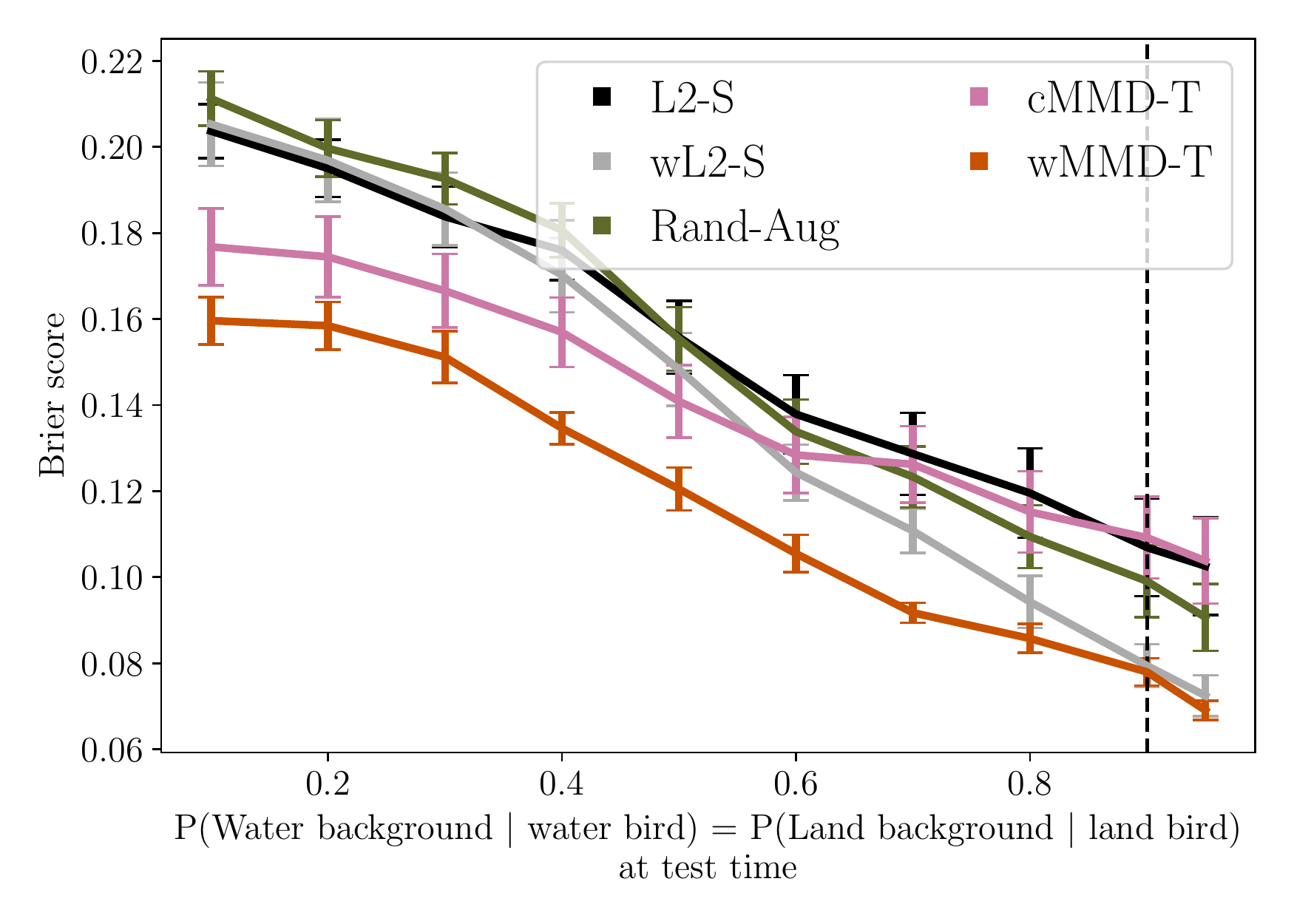}
\caption{
$x$-axis shows $P(Y|V)$ at test time under different shifted distributions, $y$-axis shows the Brier score on test data, and vertical dashed line shows $P(Y|V)$ at training time.
\label{fig:brierscore_appendix}}
\end{figure}

\subsection{Waterbirds: Rex results}
In this section we show the waterbirds results including Rex, which was excluded from the main paper because it significantly underperforms compared to other models. The poor performance of Rex is likely due to issues. First, it does not incorporate the weighting scheme and hence the invariance penalty is inconsistent with the causal DAG. Second, computing the objective function of Rex requires slicing the data into four different groups. Hence, it inherits the same unstable training dynamics that cMMD-T has (see section~\ref{sec:unstable_cmmd}). 

\begin{figure}[H]
\centering
\includegraphics[width=0.8\textwidth]{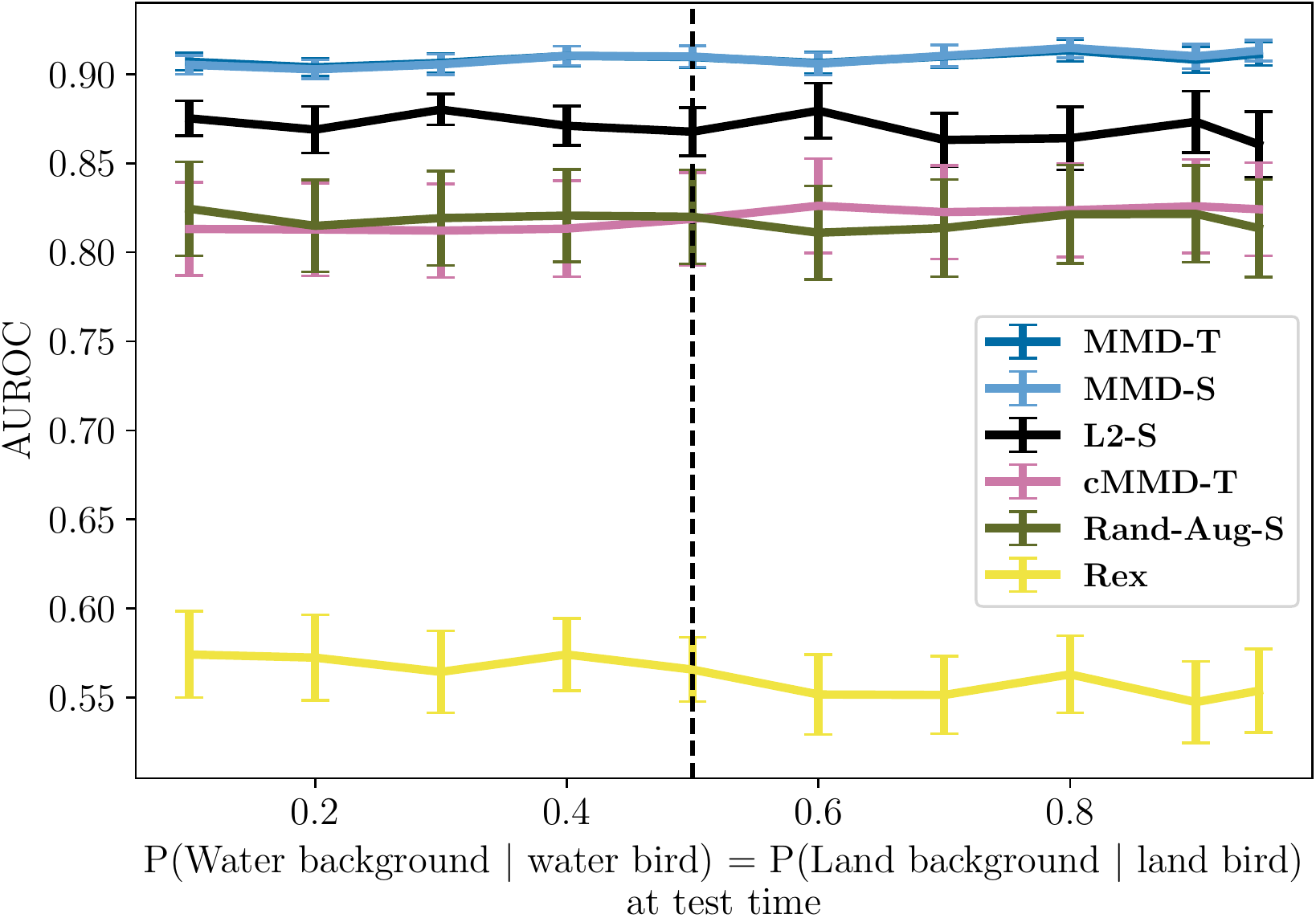}
\caption{Training data sampled from $P$, with $P(Y=1|V=1) = P^\circ(Y=0|V=0) = 0.9$. Setting similar to figure~\ref{fig:main} (left) now including results from Rex}
\end{figure}

\begin{figure}[H]
\centering
\includegraphics[width=0.8\textwidth]{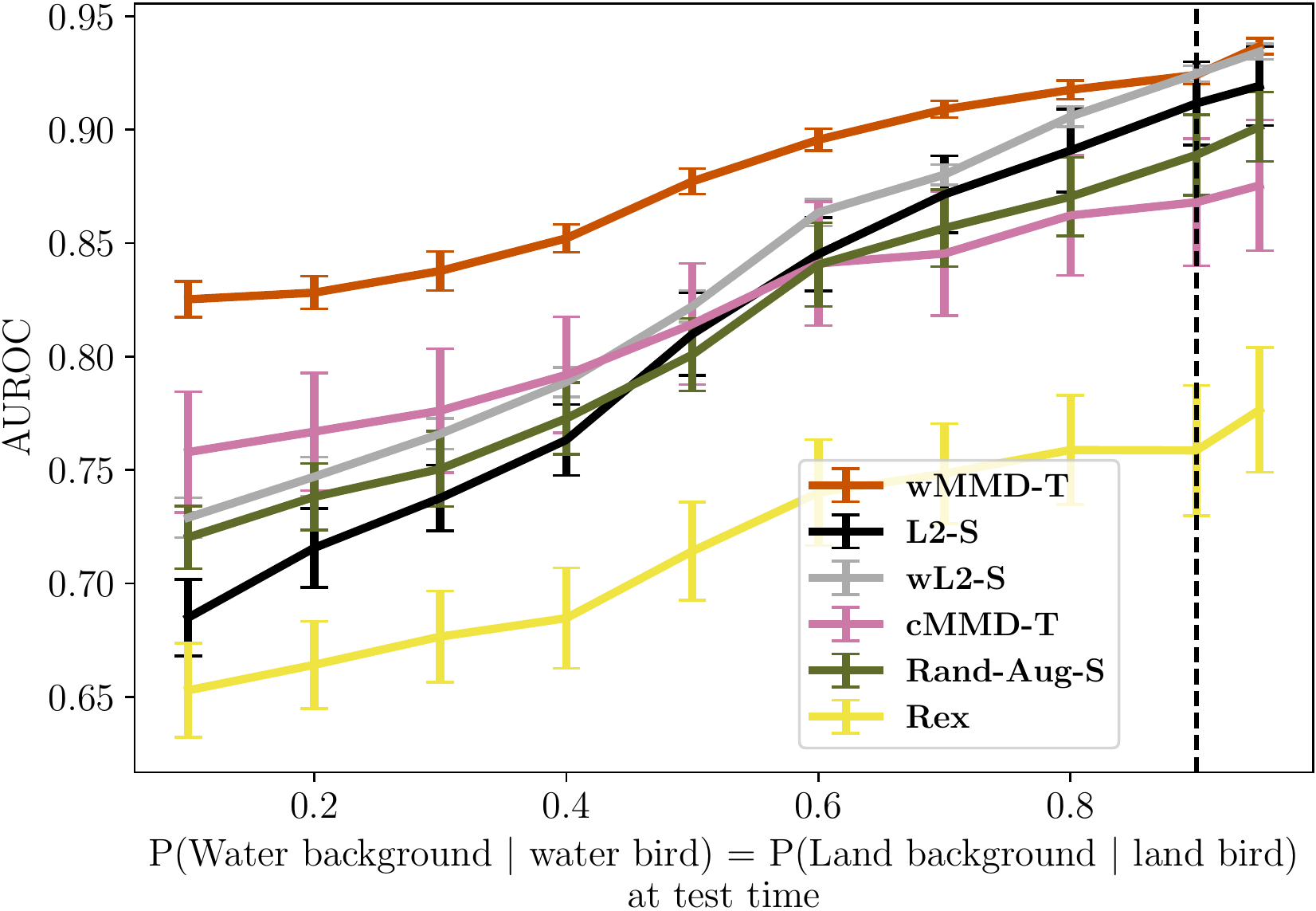}
\caption{Training data sampled from $P$, with $P(Y=1|V=1) = P^\circ(Y=0|V=0) = 0.9$. $x$. Setting similar to figure~\ref{fig:main} (middle, right) now including results from Rex}
\end{figure}

\begin{figure}[H]
\centering
\includegraphics[width=0.8\textwidth]{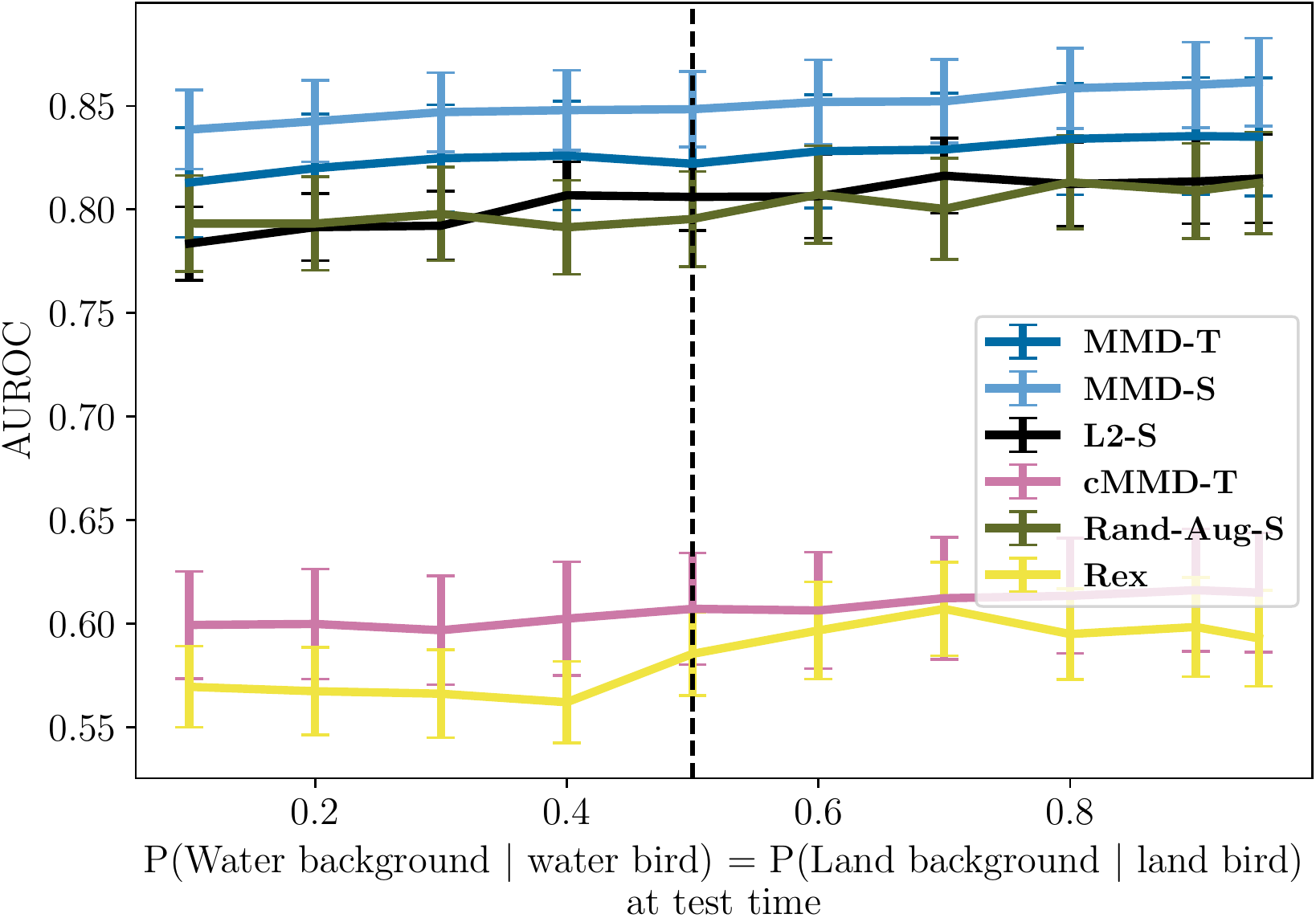}
\caption{Training data sampled from $P$, with $P(Y=1|V=1) = P^\circ(Y=0|V=0) = 0.9$, and backgrounds are sampled from a noisy set of images. Results now show the performance or Rex. Setting similar to~\ref{fig:main_a_app}, now showing the performance of Rex.}
\end{figure}

\begin{figure}[H]
\centering
\includegraphics[width=0.8\textwidth]{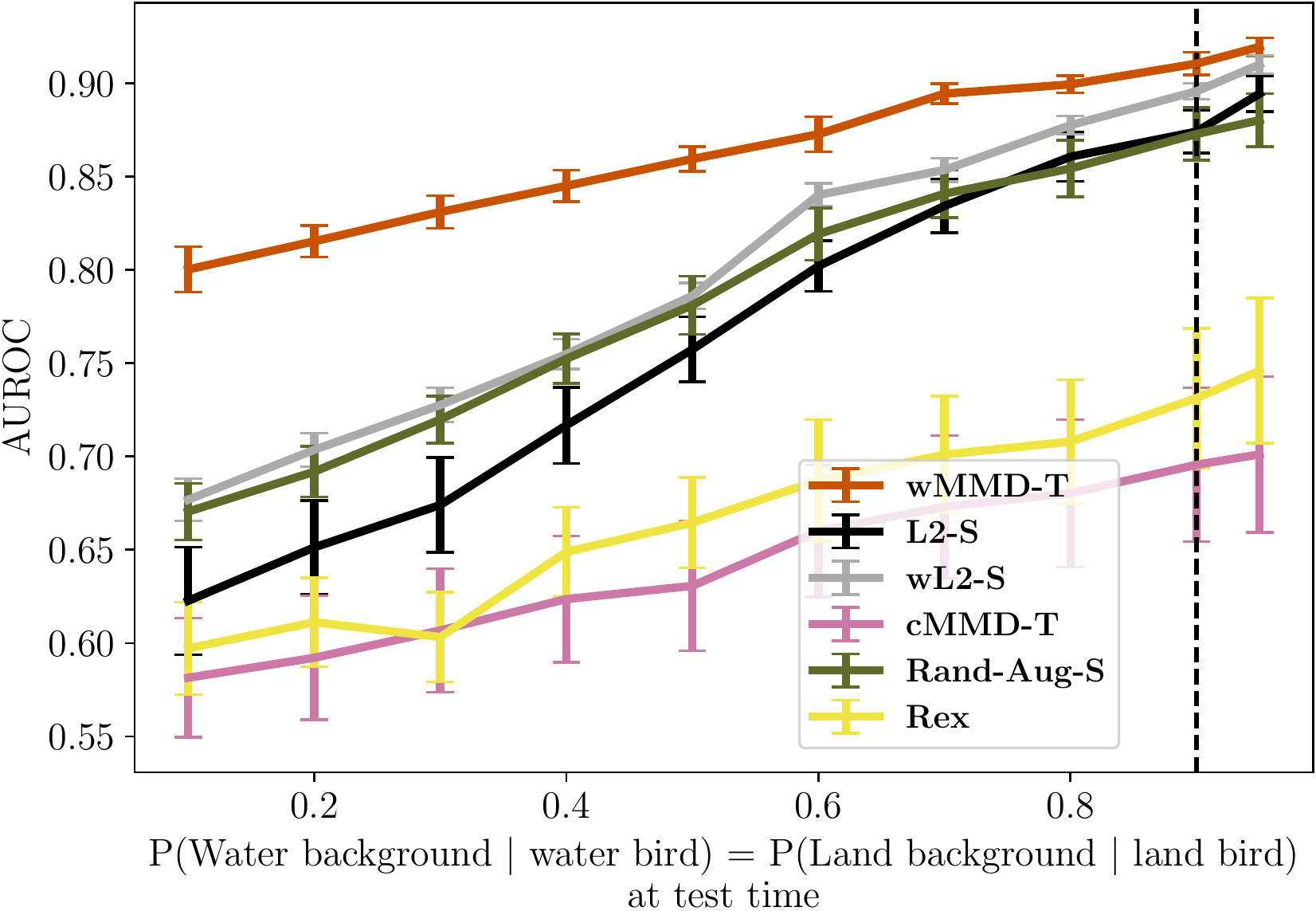}
\caption{Training data sampled from $P$, with $P(Y=1|V=1) = P^\circ(Y=0|V=0) = 0.9$. $x$, and backgrounds are sampled from a noisy set of images.  Results now show the performance or Rex. Setting similar to~\ref{fig:main_b_app}, now showing the performance of Rex.}
\end{figure}

\subsection{Training dynamics of cMMD-T}\label{sec:unstable_cmmd}
Note that the cMMD-T objective is
\begin{align}\label{eqn:eo_objective}
    h_{EO}, \phi_{EO}  = \underset{h, \phi}{\text{argmin}}& \frac{1}{n} \sum_i \ell(h(\phi(\bx_i)), y_i) +  \alpha \cdot \Big[
    \widehat{\mmd}^2 (P_{\boldsymbol{\phi}_{0,0}}, P_{\boldsymbol{\phi}_{1, 0}}) +  
     \widehat{\mmd}^2 (P_{\boldsymbol{\phi}_{0,1}}, P_{\boldsymbol{\phi}_{1, 1}})
    \Big], 
\end{align}
where $P_{\boldsymbol{\phi}_{v, y}} = P(\phi(\bX) | V = v, Y=y)$.
This differs from the objective of our main approach which does not compute the $\mmd$ penalty conditional on $Y$. This means that the cMMD-T objective requires ``slicing'' the data into smaller subgroups and computing this $\mmd$ on those smaller subgroups leading to less stable training especially in the context of DNNs, where minibatched SGD-based training is standard. For example, if the training is being done on a batch size of 16, with $P_s(Y=1) = 0.1 $, the second $\mmd$ term in equation~\eqref{eqn:eo_objective} will be computed over a sample size of roughly 2 making it a likely unreliable estimate. 

Figure~\ref{fig:waterbirds_batch_size_acc} compares the performance of our approach and cMMD-T as the batch size increases in terms of accuracy, whereas figure~\ref{fig:waterbirds_batch_size_auc} considers the AUROC. The accuracy plot shows that, for example, at a batch size of 16, our approach vastly outperforms cMMD-T. The AUROC plot shows that the cMMD-T has higher variance at every batch size.

\begin{figure}[H]
\centering
\includegraphics[width=0.8\textwidth]{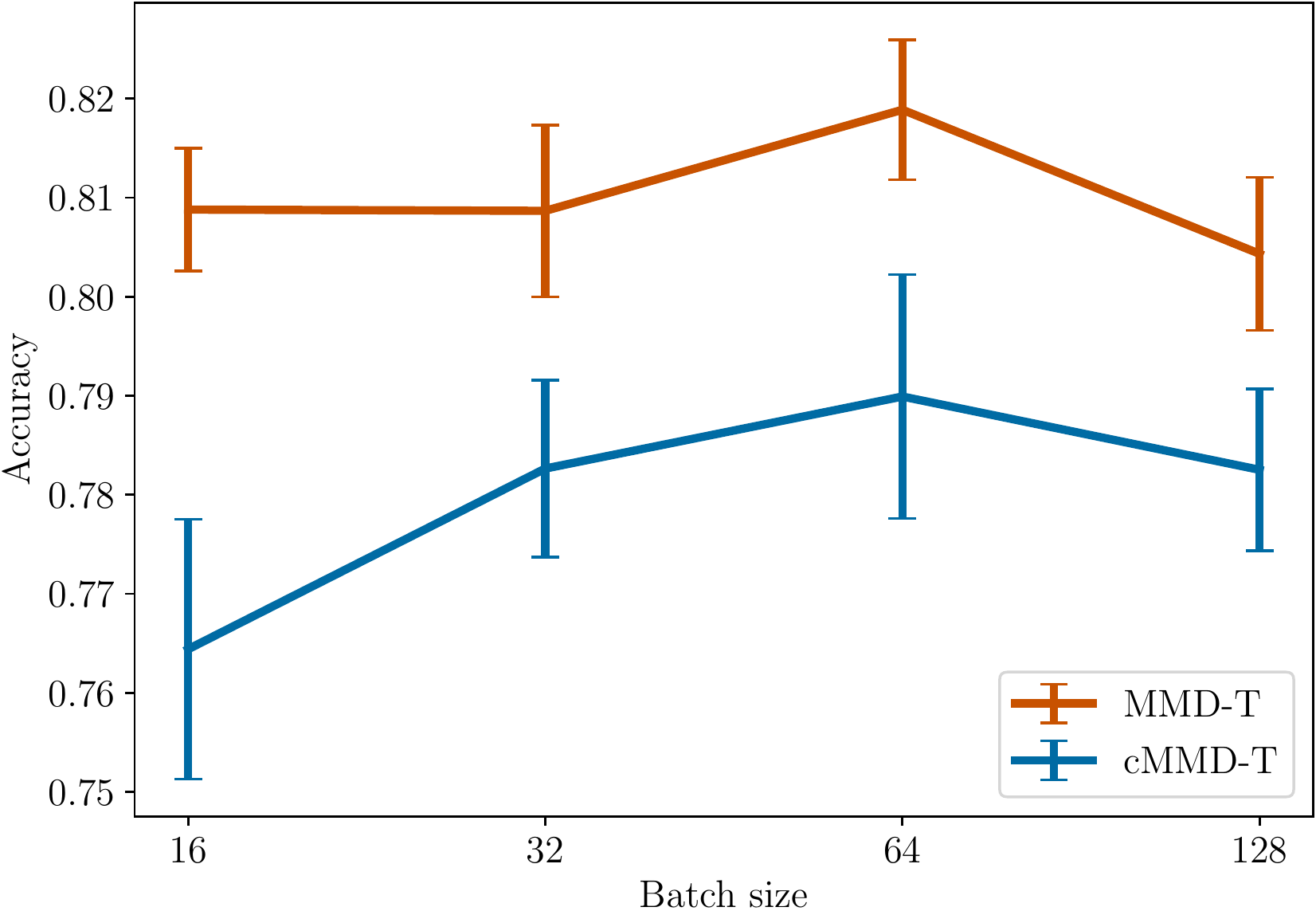}
\caption{Waterbirds data: $x$-axis shows the batch size, $y$-axis shows the accuracy. cMMD-T has poor performance in small batch sizes due to sample splitting.\label{fig:waterbirds_batch_size_acc} }
\end{figure}

\begin{figure}[H]
\centering
\includegraphics[width=0.8\textwidth]{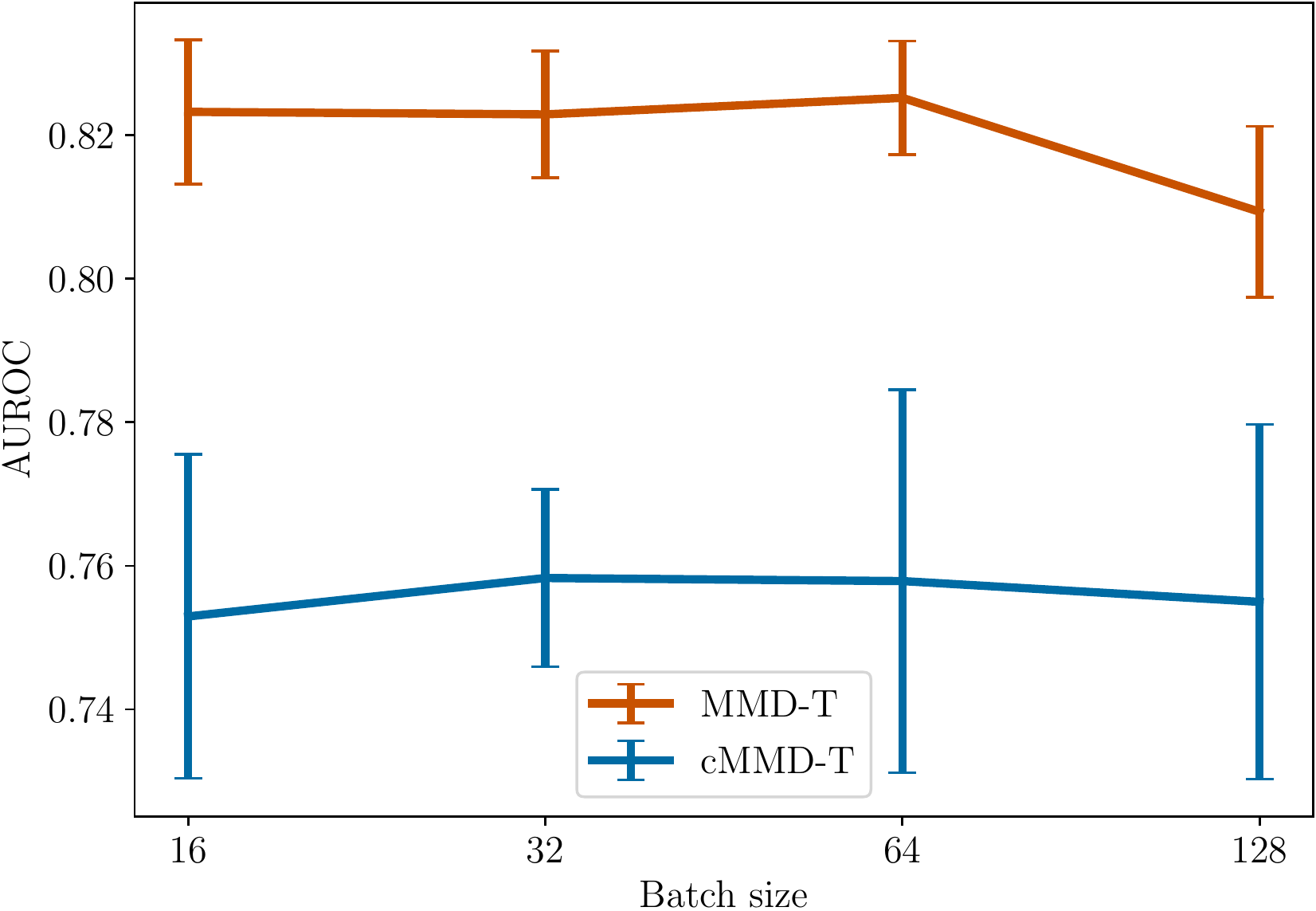}
\caption{Waterbirds data: $x$-axis shows the batch size, $y$-axis shows the AUROC. cMMD-T has an overall higher variance due to sample splitting. \label{fig:waterbirds_batch_size_auc}}
\end{figure}

\end{document}